\pdfoutput=1
\documentclass[11pt]{article} % For LaTeX2e
\usepackage{times}
\usepackage[letterpaper, left=1in, right=1in, top=1in, bottom=1in]{geometry}

\usepackage[utf8]{inputenc} % allow utf-8 input
\usepackage[T1]{fontenc}    % use 8-bit T1 fonts
\usepackage{url}            % simple URL typesetting
\usepackage{booktabs}       % professional-quality tables
\usepackage{amsfonts}       % blackboard math symbols
\usepackage{nicefrac}       % compact symbols for 1/2, etc.
\usepackage{microtype}      % microtypography
\usepackage{xspace}
\usepackage{amsmath}
\usepackage{amssymb, amsthm} % Fonts, theorems, symbols

\usepackage{algorithm,algorithmic}
\usepackage{bbm}
\usepackage{color}
\usepackage{enumitem}
\usepackage{comment}
\usepackage{bm}
\usepackage{graphicx}

\usepackage{apptools}
\usepackage[page, header]{appendix}
\usepackage{titletoc}

\usepackage[scaled=.9]{helvet}

\usepackage{url}
\usepackage{authblk}
\usepackage{caption}

\usepackage[dvipsnames]{xcolor}
\usepackage[colorlinks=true, linkcolor=blue, citecolor=blue]{hyperref}
\usepackage{thmtools}
\usepackage{thm-restate}

\usepackage{natbib}
\bibpunct{(}{)}{;}{a}{,}{,}

\renewcommand{\cite}{\citep}

\newenvironment{mythm}[2][Theorem]{\begin{trivlist}
\item[\hskip \labelsep {\bfseries #1}\hskip \labelsep {\bfseries #2}]}{\end{trivlist}}

\def\R{{\mathbb{R}}}
\def\pr{{\rm Pr}}
\def\E{{\mathbb E}}

\def\Y{{\mathcal Y}}
\def\Z{{\mathcal Z}}
\def\A{{\mathcal A}}
\def\H{{\mathcal H}}
\def\HH{{\mathcal H}}
\def\F{{\mathcal F}}
\def\G{{\mathcal G}}

\def\D{{\mathcal D}}

\def\I{{\mathcal I}}

\def\N{{\mathcal N}}

\def\ind{ \mathbbm{1}}
\def\sign{\textnormal{sign}}
\def\avg{\textnormal{avg-diam}}
\DeclareMathOperator*{\argmax}{arg\,max}

\newcount\Comments  % 0 suppresses notes to selves in text
\definecolor{darkgreen}{rgb}{0,0.5,0}
\definecolor{darkred}{rgb}{0.7,0,0}
\definecolor{teal}{rgb}{0.3,0.8,0.8}
\definecolor{orange}{rgb}{1.0,0.5,0.0}
\definecolor{purple}{rgb}{0.8,0.0,0.8}
\newcommand{\kibitz}[2]{\ifnum\Comments=1{\textcolor{#1}{\textsf{\footnotesize #2}}}\fi}

\newtheorem{thm}{Theorem}
\newtheorem{lemma}[thm]{Lemma}
\newtheorem{cor}[thm]{Corollary}

\newtheorem{defn}[thm]{Definition}

\newtheorem{assump}{Assumption}

\usepackage{authblk}

\title{Diameter-based Interactive Structure Discovery}
\author[1]{Christopher Tosh\thanks{c.tosh@columbia.edu}}
\author[1]{Daniel Hsu\thanks{djhsu@cs.columbia.edu}}
\affil[1]{Columbia University, New York, NY}

\begin{document}

\maketitle

\begin{abstract}
We introduce \emph{interactive structure discovery}, a generic framework that encompasses many interactive learning settings, including active learning, top-$k$ item identification, interactive drug discovery, and others. We adapt a recently developed active learning algorithm of \citet{TD17} for interactive structure discovery, and show that the new algorithm can be made noise-tolerant and enjoys favorable query complexity bounds.
\end{abstract}
\section{Introduction}

Standard approaches to learning structures from data generally do not incorporate human interaction into the learning process. Typically, a data set is collected and labeled, if appropriate, and an algorithm is run to find the structure that best fits the data. \emph{Interactive structure learning}, by contrast, adaptively solicits feedback from a human, or other information source, during the structure learning process. The hope is that by incorporating interaction into the learning process, we may be able to learn higher quality structures with less data or lower computational costs.

Recently, there has been interest in designing algorithms for interactive structure learning. Some works~\citep{EzK17, TD18} have attacked this problem in broad generality, designing algorithms that are capable of interactively learning generic classes of structures. Others have designed structure-specific interactive learning algorithms in a variety of settings, including flat and hierarchical clustering~\citep{WC00, ABV14, AKB16, VD16}, topic modeling~\citep{HBSS14, LCSBg17}, and matrix completion~\citep{KS14}. In all of these works, the ultimate goal is to find the structure that a user has in mind, and the algorithms are designed around this objective.

However, users of interactive learning algorithms are not always primarily interested in obtaining high-quality estimates of a particular structure. In some settings, especially those where actions are to be taken based on what has been learned, the goal is to glean information on some aspect of a structure. In information retrieval, for example, knowing the correct ranking of a set of items is often less important than getting the ordering of the first few elements correct~\citep{MSE17, SW17}. 

%Consider the setting of drug discovery~\cite{BCSVMKWLKS12, YSGELFBBST12}, where there are $n$ cell lines and $p$ drugs under consideration, along with some low-rank matrix $A \in \R^{n \times p}$, where the entry $A_{ij}$ corresponds to the survival rate of the $i$-th cell line when exposed to the $j$-th drug. Initially, all the entries of $A$ are unknown, and an experiment must be run to observe an entry. The goal here is to find the column $A_j$ that best meets some criteria, measured by a score function $s: \R^p \rightarrow \R$, while running as few experiments as possible. In this setting, although the drug that we select may depend on our estimate of $A$, our primary interest is not in obtaining a high-quality estimate of $A$, rather we only care about finding a good drug.

In this work, we introduce \emph{interactive structure discovery}, a general framework that encompasses both traditional interactive structure learning and other scenarios that have objectives which deviate from the structure estimation problem. We also demonstrate that there is a natural, general-purpose algorithm for this setting, and we give guarantees on its consistency and convergence rates, even in the presence of noise. %Finally, we conduct simulations demonstrating this algorithm's good performance.

%{\color{blue}Not sure what to say here...} A key difficulty that arises in interactive learning is the fact that data is not independent and identically distributed. The bias induced by an algorithm's choices make even simple guarantees such as consistency are far from guaranteed. To address this issue, our analysis tracks the expected behavior of our algorithm after each round of interaction, demonstrating that on average it makes progress towards the objective. 

\subsection{Paper organization}

In Section~\ref{section: interactive structure discovery}, we introduce the problem of interactive structure discovery, and provide several examples illustrating the breadth of its potential applications. In Section~\ref{section: DBAL}, we introduce an algorithm, a generalization of the \textsc{dbal} algorithm~\citep{TD17}, for the interactive structure discovery problem. In Section~\ref{section: theoretical guarantees}, we show that this algorithm is consistent and enjoys fast rates of convergence under certain conditions. We also demonstrate nearly matching lower bounds. In Section~\ref{section: examples}, we provide concrete, worked examples of these theoretical guarantees. In particular, we illustrate the improvements that interactive structure discovery can offer over other schemes that focus purely on the standard structure estimation problem. We conclude in Section~\ref{section: experiments} with simulations demonstrating that the algorithms discussed here can be practically implemented and perform well on simulated data.

\section{Interactive structure discovery}
\label{section: interactive structure discovery}

{There are a variety of settings in which adaptively solicited interaction has been shown to decrease the statistical or computational resources required for a learning problem.} In active learning, for example, algorithms that are able to adaptively query data points for their labels are able to find low-error classifiers with fewer labels than learning algorithms presented with random labels~\citep{D05, BHW10, H11}. In adaptive matrix completion, learners that adaptively query the entries of some unknown low-rank matrix are able to reconstruct the matrix with fewer revealed entries than can be done with randomly sampled entries~\citep{KS14}. In clustering, soliciting constraints from a user or oracle can improve the quality of the clustering~\citep{VD16} and circumvent computational hardness results~\citep{AKB16}. 

The examples above can be thought of as structure estimation problems -- problems where the learner's objective is to estimate some ground-truth structure. However, there are also learning situations that can benefit from interaction but are not easily framed as structure estimation problems. In the top-$k$ item identification problem, a learner queries the relative preferences of a user over a set of $n$ items with the goal of finding the $k$ most preferred items. While this problem can be solved by estimating a user's entire preference ordering, algorithms designed specifically for the top-$k$ item identification problem can get away with fewer queries~\citep{MSE17}. Another interactive learning situation that is not so cleanly expressed as a structure estimation problem is the drug discovery problem~\citep{BCSVMKWLKS12, YSGELFBBST12}, which is much like the adaptive matrix completion problem except the goal is not to estimate the entire drug-cell interaction matrix, but rather it is to find a drug exhibiting certain properties.

In this section, we formalize the problem of \emph{interactive structure discovery} which generalizes all of the above interactive learning settings into a single framework. Later, we will present a natural algorithm that operates within this general framework.

\subsection{Structure decompositions}

Denote by $\G$ the space of structures under consideration, these could be, for example, binary classifiers, or clusterings of some fixed data set, or low rank $n \times p$ matrices. Following \citet{TD18}, we view each structure in $\G$ as a function from a set of atomic questions $\A$ to a set of responses $\Y$. As the following examples illustrate, this view admits a wide spectrum of admissible structures.

\begin{itemize}
	\item \textbf{Binary classifiers}. When $\G$ is a collection of classifiers, each atom $a \in \A$ corresponds to a data point and $\Y = \{0,1\}$.
	\item \textbf{Clusterings}. If $\G$ is a set of clusterings of a collection of $n$ items, then we may view $g \in \G$ as the function from $\A = {[n] \choose 2}$ to $\Y = \{0, 1\}$, where $g((i,j))$ is 1 if $i,j$ belong to the same cluster in $g$ and 0 otherwise.
	\item \textbf{Binary hierarchical clusterings}. If $\G$ is a set of binary hierarchies over $n$ items, then we may view $g \in \G$ as the function from $\A = {[n] \choose 3}$ to $\Y = \{0, 1, 2\}$, where
\[ g((i,j,k)) = 
\begin{cases} 
0 & \text{ if } i,j \text{ are clustered before } k \text{ in } g \\
1 & \text{ if } i,k \text{ are clustered before } j \text{ in } g \\
2 &  \text{ if } j,k \text{ are clustered before } i \text{ in } g
\end{cases} \]
	\item \textbf{Matrices}. If $\G$ is a set of $n \times p$ matrices, then $\A = [n] \times [p]$ and $\Y = \R$, and $g((i,j))$ is the $(i,j)$-th entry of the matrix corresponding to $g$.  
%	\item \textbf{Metrics}. If $\G$ is a set of metrics over $n$ items, then $\A = [n] \times [n]$ and $\Y = \R_{\geq 0}$, and $g((i,j))$ is the distance between items $i$ and $j$ in metric $g$.  
\end{itemize}

We will assume that there is some distribution $\D$ over $\A$. In the case of classifiers, $\D$ is the data distribution. For clusterings over a fixed collection of items or matrices of a fixed size, a reasonable choice for $\D$ would be the uniform distribution over $\A$.

\subsection{Structure distances}

We are interested in settings where the goal may not be to recover a particular structure but perhaps only to recover some aspect of that structure. We capture this objective in the form of a \emph{structure distance} $d: \G \times \G \rightarrow \R_{\geq 0}$, which we assume to be positive, symmetric, and satisfy $d(g,g) = 0$ for all $g \in \G$. In particular, we do not require this structure distance to satisfy the triangle inequality. If $g^* \in \G$ is a ground-truth structure, then our objective is to find a structure $g \in \G$ such that $d(g, g^*)$ is small. We illustrate the flexibility of this approach with some examples.
\begin{itemize}
	\item \textbf{Low-error classifiers}. If our objective is to find a classifier with low error, then we make take our distance to be
		\[ d(g, g') \ = \ \pr_{a \sim \D}(g(a) \neq g'(a)). \]
		A classifier $g$ satisfying $d(g, g^*) < \epsilon$ will have error less than $\epsilon$. For this reason, this is the standard classification distance used to learn low-error classifiers in active learning. More generally, this is a reasonable notion of distance if our goal is to learn a high quality structure~\citep{TD18}.
	\item \textbf{Fair classifiers}. A recently proposed notion of fairness, called equal opportunity~\citep{HPS16}, attempts to balance the number of true positives between individuals with a certain protected attribute and those without the protected attribute. If our goal is to find a classifier that approximately satisfies this notion of fairness while simultaneously achieving low error, then we may take $ d(g,g')$ to be
%	\begin{align*}
%	&  \max \{   \pr_{a \sim \D}(g(a) \neq g'(a)),  \\
%	 & \lambda |\pr_{a \sim \D_0}(g(a) = 1 | g'(a) =1) \\
%	 &\hspace{6em} - \pr_{a \sim \D_1}(g(a) = 1 | g'(a) =1)|, \\
%	 & \lambda |\pr_{a \sim \D_0}(g'(a) = 1 | g(a) =1) \\
%	 &\hspace{6em} - \pr_{a \sim \D_1}(g'(a) = 1 | g(a) =1)| \}
%	\end{align*}
	\begin{align*}
	d(g,g') = & \max \{   \pr_{a \sim \D}(g(a) \neq g'(a)),  \\
	 \ & \ \ \lambda |\E_{a \sim \D_0}[g(a) | g'(a) =1]  - \E_{a \sim \D_1}[g(a)| g'(a) =1]|, \\
	 \ &  \ \ \lambda |\E_{a \sim \D_0}[g'(a) | g(a) =1]  - \E_{a \sim \D_1}[g'(a) | g(a) =1]| \}
	\end{align*}
	where $D_p$ denotes the distribution of a point conditioned on it having protected attribute value $p$ and $\lambda > 0$ is some weight of the relative importance of fairness. If we find a $g$ satisfying $d(g, g^*) < \epsilon$, then we know that the error of $g$ is at most $\epsilon$ and we violate equal opportunity by at most $\epsilon/\lambda$.
	\item \textbf{Cluster identification}. In certain clustering situations, there is some particular item of interest $i^*$, and our goal is to find the cluster to which $i^*$ belongs. In this case, we may take $d(g,g')$ to be
%	\begin{align*}
%	d(g,g') \ & = \ \max \left\{  \frac{|C(g, i^*) \setminus C(g', i^*)|}{|C(g, i^*)|}, \right. \\
%	& \hspace{4em} \left. \frac{|C(g', i^*) \setminus C(g, i^*)|}{|C(g', i^*)|} \right\}
%	\end{align*}
\[ d(g,g') \ = \ \max \left\{  \frac{|C(g, i^*) \setminus C(g', i^*)|}{|C(g, i^*)|}, \frac{|C(g', i^*) \setminus C(g, i^*)|}{|C(g', i^*)|} \right\} \]
	where $C(g,i) = \{ j \in [n] \, : \, g((i,j)) = 1 \}$ is the set of items in the same cluster as $i$ under $g$. If we find a $g$ satisfying $d(g, g^*) < \epsilon$, then we know that $C(g,i^*)$ is missing at most an $\epsilon$ fraction of the elements of $C(g^*,i^*)$ and at most an $\epsilon$ fraction of $C(g,i^*)$ is not included in $C(g^*,i^*)$.
	\item \textbf{Column selection}. If our goal is to find the best column of an $n \times p$ matrix as measured by some score function $s: \R^n \rightarrow \R$, then we may define our distance as
	\begin{align*}
	 d(g, g') \ = \ \max \{  s(g(\cdot, j_g)) - s(g(\cdot, j_{g'})),  s(g'(\cdot, j_{g'})) - s(g'(\cdot, j_{g})) \}
	\end{align*}
 	where $g(\cdot, j)$ denotes the $j$th column of $g$ and $j_g = \argmax_{j} s(g(\cdot, j))$. If we find a $g$ satisfying $d(g, g^*) < \epsilon$ and select column $j_g$, then the true score of $j_g$ is at most $\epsilon$ worse then the true score of the best column.
\end{itemize}

As the preceding examples show, the structure distance is a flexible way to encode objectives into the structure discovery problem. Throughout the remainder of the paper, we will assume that we have such a distance $d(\cdot, \cdot)$, that our objective is to find a $g \in \G$ satisfying $d(g,g^*) < \epsilon$ for some $\epsilon >0$, and that we can efficiently compute $d(g,g')$ for any two structures $g, g' \in \G$. We will also assume that $d(g,g') \leq 1$, which can be achieved with an appropriate normalization.

\section{Diameter-based structure discovery}
\label{section: DBAL}

Given a set of structures $\G$ and a suitable distance, how do we find a structure with low distance to the ground truth? One approach, which \citet{TD17} proposed for the realizable binary classification setting, is to try to find a distribution over $\G$ such that structures are close to $g^*$ on average. We take up their approach again here in our more general and potentially noisy setting.

Let $\pi$ be some probability measure over $\G$. Define the \emph{average diameter} of $\pi$ as
\[ \avg(\pi) \ = \ \E_{g, g' \sim \pi}[d(g, g')]. \]
The following result, due to~\citet{TD17}, shows that if one can find a distribution $\pi$ with low average diameter that puts sufficient mass on a target structure $g^*$, then one can readily find a structure with small distance to $g^*$ by random sampling.

\begin{lemma}
\label{lem: average diameter guarantee}
If $g^* \in \G$ and $\pi$ is a distribution over $\G$, then $\E_{g \sim \pi}[d(g, g^*)] \leq {\avg(\pi)}/{\pi(g^*)}.$
\end{lemma}
Although Lemma~\ref{lem: average diameter guarantee} was originally stated for the case where $d(\cdot,\cdot)$ is the disagreement probability of two classifiers, it still holds in our setting.

Lemma~\ref{lem: average diameter guarantee} reduces the problem of finding a structure close to $g^*$ to that of finding a distribution $\pi$ with low average diameter, provided we can sample from it. Thus, we are interested in queries whose answers will help us find distributions with low average diameter. This motivates the concept of average splitting. 

For any subset $V \subset \G$, let $\pi|_V$ denote the conditional distribution of $\pi$ restricted to $V$. For a given atom $a \in \A$ and a possible response $y \in \Y$, let $\G_a^y = \{ g \in \G \, : \, g(a) = y \}$ denote the set of structures consistent with $y$ on atom $a$. For any $a \in \A$, we say that \emph{$a$ $\rho$-average splits} $\pi$ if
\begin{equation}
\label{eqn: average splitting definition}
\hspace{-0.08em}\max_{y \in \Y} \pi(\G_a^y)^2 \, \avg( \pi|_{\G_a^y})  \leq  (1-\rho) \, \avg(\pi) 
\end{equation}
We say that $\pi$ is \emph{$(\rho, \tau)$-average splittable} if the probability that a random $a$ drawn from $\D$ $\rho$-average splits $\pi$ is at least $\tau$; and
%\[ \pr_{a \sim \D}\left( a \, \rho\text{-average splits } \pi \right) \ \geq \ \tau. \]
we say that $\G$ has \emph{average splitting index} $(\rho, \epsilon, \tau)$ if any distribution $\pi$ over $\G$ satisfying $\avg(\pi) > \epsilon$ is $(\rho, \tau)$-average splittable.

Given an efficient sampler for $\pi$, we can estimate all of the relevant quantities in equation~\eqref{eqn: average splitting definition} via Monte Carlo approximations: if $g, g'$ are drawn i.i.d. from $\pi$ then for any $a \in \A$ and $y \in \Y$,
\[ \E[d(g,g') \ind[g(a) = y = g'(a)]]  =  \pi(\G_a^y)^2 \avg(\pi|_{\G_a^y}). \]

\subsection{Finding a good query}
Suppose that we want to choose from a set of atoms the one that provides the largest average split, say $\rho$, of $\pi$. How do we go about doing this? In the case where $\G$ is a binary hypothesis class and $\avg(\pi)$ has a known lower bound $\epsilon$, \citet{TD17} gave an algorithm that can find a query that $O(\rho)$-average splits $\pi$ while sampling $\tilde{O}(1/(\epsilon \rho^2) + 1/\avg(\pi)^2)$\footnote{The $\tilde{O}(\cdot)$ suppresses logarithmic factors in $1/\delta$ and the number of candidate atoms.} structures from $\pi$. 

In Algorithm~\ref{algorithm: select algorithm}, we present an algorithm based on inverse sampling~\citep{H45} that enjoys the same guarantees in a more general setting while sampling fewer structures.

\begin{restatable}{lemma}{SelectLemma}
\label{lem: select lemma}
Pick $\alpha, \delta > 0$. If \textsc{select} is run with atoms $a_1, \ldots, a_m$, one of which $\rho$-average splits $\pi$, then with probability $1-\delta$, \textsc{select} returns a data point that $(1-\alpha)\rho$-average splits $\pi$ while sampling no more than 
\[ \frac{12}{\alpha^2(1-\alpha) \rho \, \avg(\pi)} \log \frac{m + |\Y|}{\delta} \]
pairs of structures in total.
\end{restatable}

%With high probability, the running time of \textsc{select} is $O\left( \frac{T_{\text{sample}} + m |\Y|}{\alpha^2(1-\alpha) \rho \, \avg(\pi)} \log  \frac{m + |\Y|}{\delta} \right)$, where $T_{\text{sample}}$ is the time needed to sample a structure.

We defer all of the proofs in the paper to the appendix, but we sketch the main intuition of \textsc{select} here. Say that $g, g' \sim \pi$. The key observation is that each atom $a_i$ has an associated average split $\rho_i$ such that for any response $y$, 
\[  \E[d(g,g')(1 - \ind[g(a_i) = y = g'(a_i)])]  \geq  \rho_i \, \avg(\pi)  \]
and moreover there exists some response $y^*$ such that
\[  \E[d(g,g')(1 - \ind[g(a_i) = y^* = g'(a_i)])]  =  \rho_i \, \avg(\pi) . \]
Suppose that we choose $N$ and draw $g_j, g'_j \sim \pi$ sequentially until a round $K_i$ in which all $y \in \Y$ satisfy 
\[ S^{a_i, y}_{K_i} = \sum_{j=1}^{K_i} d(g_j,g_j')(1 - \ind[g_j(a_i) = y = g_j'(a_i)]) \geq N.  \]
Then one can show that $K_i$ is tightly concentrated around $\frac{N}{ \rho_i \, \avg(\pi)}$~\citep{H45}. Thus, the first atom $a_i$ to satisfy that $S^{a_i, y}_{K} \geq N$ is likely to satisfy that $\rho_i \geq (1-\alpha) \max_{j} \rho_i$ for some constant $\alpha$ and the number of rounds needed for this to happen will satisfy $K \approx \frac{N}{ \rho_i \, \avg(\pi)}$.

\subsection{Noise-tolerant DBAL}
%-----------------------------0-1 Loss NDBAL -----------------------------
\begin{figure*}[ttt!]
\begin{minipage}[t]{0.46\textwidth}
\begin{algorithm}[H]
 \caption{{\textsc{ndbal}}}
 \label{algorithm: 0-1 loss algorithm}
\begin{algorithmic}
   \STATE {\bfseries Input:} Distribution $\pi$, $\beta > 0$, $\alpha, \delta \in (0,1)$
   \STATE Initialize $\pi_o = \pi$
   \FOR{$t=1,2,\ldots$}
   		\STATE \hspace{-0.5em}Draw $m$ atoms $\mathbf{a} = (a_1, \ldots, a_m)$
		\STATE \hspace{-0.5em}Query $a_t = \textsc{select}(\pi_{t-1}, \mathbf{a}, \alpha, \delta)$ and receive $y_t$ 
		\STATE \hspace{-0.5em}$\pi_{t}(g) \propto \pi_{t-1}(g)\exp \left(-\beta \ind[g(a_t) \neq y_t)] \right)$
	\ENDFOR
	\STATE \textbf{return} Posterior $\pi_t$
\end{algorithmic}
\end{algorithm}
\end{minipage}
\hfill
\begin{minipage}[t]{0.55\textwidth}
\begin{algorithm}[H]
 \caption{{\textsc{select}}}
 \label{algorithm: select algorithm}
\begin{algorithmic}
   \STATE {\bfseries Input:} Distribution $\pi$, atoms $a_1, \ldots, a_m$
   \STATE Set $N = \frac{6(2+\alpha)}{\alpha^2} \ln \frac{m + |\Y|}{\delta}$, $K = 0$, $S^{a_i, y}_0 = 0$ 
   \FOR{$K=1,2,\ldots$}
     \STATE Draw $g, g' \sim \pi$ and compute for all $a_i, y$: \\ \vspace{0.5em}
     $\hspace{1em}S^{a_i, y}_K = S^{a_i, y}_{K-1} + d(g,g')(1 - \ind[g(a_i) = y = g'(a_i)])$
     \STATE \vspace{-0.39em} If $\exists a_i$ s.t. $S^{x_i, y}_K \geq N$ for all $y \in \Y$, \textbf{halt} and \textbf{return} $a_i$.
  \ENDFOR
\end{algorithmic}
\end{algorithm}
\end{minipage}
\end{figure*}

The approach of \citet{TD17} was to maintain a distribution $\pi_t$ over all structures that are consistent with the feedback observed so far. In our setting, this corresponds to the posterior update rule
\begin{equation}
\label{eqn: noiseless update}
\pi_t(g) \ \propto \ \pi_{t-1}(g) \ind[g(a_t) = y_t]
\end{equation}
after querying $a_t$ and receiving response $y_t$. Their algorithm, termed \textsc{dbal} for Diameter-based Active Learning, was shown to have favorable query complexity dependence on the average splitting index in the noiseless and realizable binary classification setting. 

In this work, we want to be able to handle settings where our responses are noisy or inconsistent with a ground-truth structure. Following~\citet{N11}, we consider a `softer' posterior update:
\begin{equation}
\label{eqn: 0-1 loss update}
\pi_t(g) \ \propto \ \pi_{t-1}(g) \exp(- \beta \ind[g(a_t) \neq y_t])
\end{equation}
where $\beta > 0$ is some parameter corresponding roughly to our confidence in the accuracy of the responses. Note that by taking $\beta \rightarrow \infty$, we recover the update in equation~\eqref{eqn: noiseless update}. We call this algorithm \textsc{ndbal} for Noise-tolerant Diameter-based Active Learning. The full algorithm for \textsc{ndbal} is displayed in Algorithm~\ref{algorithm: 0-1 loss algorithm}.

The update in equation~\eqref{eqn: 0-1 loss update} has been shown to enjoy favorable guarantees for active learning strategies that attempt to shrink $\pi$-mass~\citep{N11, TD18}. We will show that it also works well for \textsc{ndbal}.

\section{Theoretical guarantees}
\label{section: theoretical guarantees}

In this section, we establish the statistical consistency of \textsc{ndbal} and study its rate of convergence. To do so, we need to formalize our problem set up. Note that at each time $t$, the random outcomes consist of the atom $a_t$ that we query, as well as the response $y_t$ to $a_t$. Let $\F_t$ denote the sigma-field of all outcomes up to and including time $t$. 

\subsection{Consistency}

We first show that \textsc{ndbal} is consistent, i.e. $ \E_{g \sim \pi_t}[d(g,g^*)] \rightarrow 0$ as $t \rightarrow \infty$ almost surely (a.s.), where $g^* \in \G$ is a ground truth structure. To do so, we need to make a few assumptions on our problem set up. Our first assumption is that $\G$ is finite. This will be relaxed when we study faster rates of convergence.

Our next assumption is that any two structures with positive distance can be distinguished by a random atom with positive probability. 
\begin{assump}
\label{assump: identifiable distance}
For any $g, g' \in \G$ such that $d(g,g') > 0$, we have $ \pr_{a \sim \D}(g(a) \neq g'(a)) > 0$.
\end{assump}
{Note that Assumption~\ref{assump: identifiable distance} is necessary for identifiability: when Assumption~\ref{assump: identifiable distance} does not hold, there exist  structures $g, g'$ with $d(g,g') > 0$ that cannot be distinguished with atomic questions.}

We will also need to make an assumption on the typical responses provided by a user. Let $\eta(y \, | \, a)$ denote the conditional probability of response $y$ to atomic question $a$. We will require that the most likely response to an atomic query is the true response. 

\begin{assump}
\label{assump: bounded noise}
There exist $g^* \in \G$ and $\lambda > 0$ such that $ \eta(g^*(a) \, | \, a) \geq \eta(y \, | \, a) + \lambda $ for any $a \in \A$ and $y \neq g^*(a)$. 
\end{assump}

In the setting where $\G$ is a collection of binary classifiers, Assumption~\ref{assump: bounded noise} is equivalent to Massart's bounded noise condition~\citep{ABHU15}. This noise condition has been previously studied in the active learning literature under the related notion of the splitting index~\citep[Appendix C]{BH12}, albeit with a different active learning algorithm.

Our analysis will focus on the behavior of the potential function ${\avg(\pi_{t})}/{\pi_{t}(g^*)}$. By Lemma~\ref{lem: average diameter guarantee}, whenever this potential function goes to 0, $\E_{g \sim \pi_t}[d(g,g^*)]$ also must go to 0. The following lemma demonstrates that under Assumption~\ref{assump: bounded noise}, a related potential function is guaranteed to decrease in expectation.

\begin{restatable}{lemma}{GeneralKLoss}
\label{lem: general k 0-1 loss decrease}
Pick $k \geq 2$. Suppose Assumption~\ref{assump: bounded noise} holds and $\beta \leq \lambda/(2 + 2k^2)$. If we query an atom $a_t$ that $\rho$-average splits $\pi_{t-1}$, then in expectation over the randomness of the response $y_t$, we have 
\[ \E\left[ \frac{\avg(\pi_{t})}{\pi_{t}(g^*)^k} \, \bigg| \F_{t-1}, a_t \right] = \left(1- \Delta\right) \frac{\avg(\pi_{t-1})}{\pi_{t-1}(g^*)^k}   \]
where $\Delta \geq \rho \lambda \beta/2$.
\end{restatable}

Thus, at each at each round, ${\avg(\pi_{t})}/{\pi_{t}(g^*)^k}$ decreases in expectation by a multiplicative factor of $1- \Delta$, for an appropriate choice of $\beta$. However, Lemma~\ref{lem: general k 0-1 loss decrease} does not tell us how $\avg(\pi_t)$ and $\pi_t(g^*)$ behave individually. The following lemma shows that $1/\pi_t(g^*)^k$ is a supermartingale.

\begin{restatable}{lemma}{PiMassSupermartingale}
\label{lem: pi-mass supermartingale}
Pick $k \geq 1$. Suppose Assumption~\ref{assump: bounded noise} holds and $\beta \leq \lambda/k$. Then for any query $a_t$, we have
$ \E\left[ {1}/{\pi_t(g^*)^k} \, | \, \F_{t-1}, a_t \right] \ \leq \ {1}/{\pi_{t-1}(g^*)^k} . $
\end{restatable}

Lemma~\ref{lem: general k 0-1 loss decrease} also tells us how much ${\avg(\pi_{t})}/{\pi_{t}(g^*)^k}$ decreases in expectation given that we query a point that $\rho$-average splits the current posterior. In order to demonstrate consistency, we need $\rho$ to be lower bounded on average. The following lemma gives such a lower bound for points chosen by \textsc{ndbal}.

\begin{restatable}{lemma}{ExpectedSplittingLowerBound}
\label{lem: expected splitting lower bound}
If Assumption~\ref{assump: identifiable distance} holds and \textsc{ndbal} is run with constants $\alpha, \delta \in (0,1)$, then there is a constant $c >0$, depending on $\alpha,\delta, d(\cdot, \cdot), \G$ and $\D$, such that for every round $t$, \textsc{ndbal} queries a point that $\rho_t$-average split $\pi_t$ satisfying $\E[\rho_t \, | \, \F_{t-1}] \ \geq \ \frac{c}{1 - \log (\avg(\pi_t))}. $
\end{restatable}

In the appendix, we show how the above results imply consistency for \textsc{ndbal}.

\begin{restatable}{thm}{DBALConsistency}
\label{thm: DBAL consistency}
If Assumptions~\ref{assump: identifiable distance} and~\ref{assump: bounded noise} hold, $\beta \leq \lambda/10$, and $\pi_o(g^*) > 0$, then 
$\E_{g\sim\pi_t}[d(g,g^*)] \rightarrow 0$ with probability one.
\end{restatable}

\subsection{Convergence rates}

We now turn to the setting where there is some fixed error threshold $\epsilon > 0$, and our goal is to find a distribution $\pi_t$ satisfying $\E_{g \sim \pi_t}[d(g,g^*)] \leq \epsilon$. The following theorem gives a bound on the resources that \textsc{ndbal} uses to find such a distribution.

\begin{restatable}{thm}{DBALGuarantees}
\label{thm: 0-1 loss DBAL guarantees}
Let $\epsilon, \delta  > 0$ and $\epsilon_o = \epsilon \delta \pi(g^*)/4$. If Assumption~\ref{assump: bounded noise} holds, $\G$ has average splitting index $(\rho, \epsilon_o, \tau)$ and \textsc{ndbal} is run with $\beta \leq \lambda/10$ and $\alpha=1/2$, then with probability $1-\delta$, \textsc{ndbal} encounters a distribution $\pi_t$ satisfying $\E_{g \sim \pi_t}[d(g,g^*)] \leq \epsilon$ while the resources used satisfy:
\begin{itemize}
	\item[(a)] $T \leq \frac{2}{\rho \lambda \beta (1-\beta)} \max \left( \ln \frac{1}{\epsilon \pi(g^*)^2}, \frac{2e^{2\beta}}{\rho \lambda \beta (1-\beta)} \ln \frac{1}{\delta} \right)$ rounds, with one query per round,
	\item[(b)] $m_t \leq \frac{1}{\tau} \log \frac{4t(t+1)}{\delta}$ atoms drawn per round, and
	\item[(c)] $n_t \leq O \left( \frac{1}{\rho \epsilon_o} \log \frac{(m_t + |\Y|)t(t+1)}{\delta} \right)$ structures sampled per round.
\end{itemize}
\end{restatable}

\noindent While Theorem~\ref{thm: 0-1 loss DBAL guarantees} does provide rates of convergence, it has several issues.
\begin{itemize}
	\item[(i)] The number of structures sampled in each round is polynomial in $1/\pi(g^*)$, which can be large. 
	\item[(ii)] Theorem~\ref{thm: 0-1 loss DBAL guarantees} only guarantees that some posterior we encounter will satisfy $\avg(\pi_t)/\pi_t(g^*)^2 < \epsilon$; in particular, it does not tell us how to detect \emph{which} posterior satisfies this property.
	\item[(iii)] The average splitting index $(\rho, \epsilon_o, \tau)$ depends on $\pi(g^*)$. In settings where the average splitting index has been bounded~\citep{D05, TD17}, $\rho$ and $\tau$ depend on $\epsilon_o$, implying that the query complexity and the number of atoms drawn per round grow as $\pi(g^*)$ shrinks.
\end{itemize}
Without any further assumptions, issues (i) and (iii) are unavoidable even in the noiseless setting. To see why, consider a setting in which our prior only puts mass on two structures $g$ and $g^*$ where $d(g,g^*) \approx 1$. If structures are only accessed via a sampling oracle, detecting that there are two structures with positive probability mass requires $\Omega(1/\pi(g^*))$ samples. Moreover, in this scenario we have $ \E_{g' \sim \pi}[d(g', g^*)] > \epsilon$ whenever $\avg(\pi)/\pi(g^*) > \epsilon/2$.
Thus, with no further assumptions, we need to incur computational and data complexity costs that depend on $\pi(g^*)$. 

\subsection{Faster convergence rates}

As discussed above, when $g^*$ is completely independent of our prior $\pi$, {\sc ndbal} incurs high computational and data complexity costs. We show that this is avoided under the following Bayesian assumption on $g^*$.

\begin{assump}
\label{assump: Bayesian}
There exists a $\lambda \geq 1$ and distribution $\nu$ over $\G$ such that the true structure $g^*$ is drawn from $\nu$ and $1/\lambda \leq \nu(g)/\pi(g) \leq \lambda$ for every $g \in \G$.
\end{assump}

Assumption~\ref{assump: Bayesian} is a slight relaxation of the traditional Bayesian assumption. Here we do not require $g^*$ to be drawn from $\pi$ itself, but rather only that it is drawn from some distribution that is close to $\pi$. %This assumption can be found in the label complexity analysis of the query-by-committee algorithm~\cite{FSST97}.

For ease of presentation, we also assume that we are in the completely noiseless setting. In the appendix, we show that we there is a certain amount of noise that we can tolerate and still get very fast rates of convergence. Formally, we make the following assumption.
\begin{assump}
\label{assump: noiseless}
There is a $ g^* \in \G$ such that $\eta(g^*(a) \, | \, a) = 1$.
\end{assump}
\noindent With Assumption~\ref{assump: noiseless}, we will run {\sc ndbal} with $\beta = \infty$ and get the posterior update in equation~\eqref{eqn: noiseless update}.

Together, Assumptions~\ref{assump: Bayesian} and~\ref{assump: noiseless} immediately add more structure to our setting. In particular, if we have query/response pairs $(a_1, y_1), \ldots, (a_t, y_t)$, then the true posterior takes the form
\[ \nu_t(g) \ \propto \  \nu(g) \ind[g(a_i) = y_i \text{ for } i=1,\ldots,t] . \]
Without access to $\nu$, there is no way to compute $\nu_t$ directly; however, we may still hope that a random draw from our distribution $\pi_t$ is close to a random draw from $\nu_t$, i.e. that the quantity
\[ D(\pi_t, \nu_t) \ = \ \E_{g \sim \pi_t, g^* \sim \nu_t}[d(g, g^*)] \] 
is small. Thus, our new objective is to find a distribution $\pi_t$ satisfying $D(\pi_t, \nu_t) \leq \epsilon$. Given this new objective, we relax the requirement that $\G$ is finite. Instead, we assume that $\G$ has bounded \emph{graph dimension}~\citep{B89}, a multiclass generalization of the VC dimension.
\begin{defn}
Let $S = \{a_1, \ldots, a_m \}$ be a set of atomic questions. We say $\G$ \emph{shatters} $S$ if there exists $f \, : \, S \rightarrow \Y$ such that for all $T \subset S$, there exists $g_T \in \G$ such that $g_T(x) = f(x)$ when $x \in T$ and $g_T(x) \neq f(x)$ when $x \in S \setminus T$. The \emph{graph dimension} of $\G$ is the size of the largest $S$ such that $\G$ shatters $S$.
\end{defn}

Finally, we need to decide when to stop making queries. As discussed in the previous section, one of the shortcomings of Theorem~\ref{thm: 0-1 loss DBAL guarantees} is that it gives no guidance on when we have found a good distribution $\pi_t$. To address this, we use the stopping rule suggested by \citet{TD17}: estimate $\avg(\pi_t)$ by sampling $\tilde{O}(\lambda^2/\epsilon)$ pairs of structures at the beginning of each round and stop if this estimate is below $3\epsilon/(4\lambda^2)$. Given this modification, we can improve the guarantees of \textsc{ndbal}.

\begin{restatable}{thm}{NoiselessBayesianConvergence}
\label{thm: noiseless Bayesian convergence}
Suppose $\G$ has average splitting index $(\rho, \epsilon/(2\lambda^2), \tau)$ and graph dimension $d_G$. If Assumptions~\ref{assump: Bayesian} and~\ref{assump: noiseless} hold, then with probability $1- \delta$, modified {\sc ndbal} terminates with a distribution $\pi_t$ satisfying $D(\pi_t, \nu_t) \leq \epsilon$ while using the following resources:
\begin{itemize}
	\item[(a)] $T \leq {O}\left(\frac{d_G}{\rho} \left( \log \frac{|\Y| \lambda}{\epsilon \tau \delta} + \log^2 \frac{d_G}{\rho} \right) \right)$ rounds with one query per round,
	\item[(b)] $m_t \leq O \left(\frac{1}{\tau} \log \frac{t}{\delta} \right)$ atoms drawn per round, and
	\item[(c)] $n_t \leq O \left(\left(\frac{\lambda^2}{\epsilon \rho} \right) \log \frac{(m_t + |\Y|)t}{\delta} \right)$ structures sampled per round.
\end{itemize}
\end{restatable}
\noindent In the appendix, we also consider the noisy setting.

\subsection{Lower bounds}

The results above demonstrate that the average splitting index provides upper bounds on the resource complexity of \textsc{ndbal} in this generic interactive structure discovery setting. %One might therefore ask if this dependence on the average splitting index is necessary. 
The following theorem shows that, in fact, some dependence on the average splitting index is inevitable for \emph{any} learner in this setting.

\begin{restatable}{thm}{AverageSplittingLowerBound}
\label{thm: average splitting lower bound}
Fix $\G$, $\D$ and $d(\cdot, \cdot)$. If $\G$ does not have average splitting index $(\frac{\rho}{4 \lceil \log 1/\epsilon \rceil}, 2\epsilon, \tau)$ for some $\rho,\epsilon \in (0,1)$ and $\tau \in (0,1/2)$, then any interactive learning strategy which with probability $> 3/4$ over the random sampling from $\D$ finds a structure $g \in \G$ within distance $\epsilon/2$ of any target in $\G$ must draw at least $1/\tau$ atoms from $\D$ or must make at least $1/\rho$ queries.
\end{restatable}

The proof of Theorem~\ref{thm: average splitting lower bound} is similar to the one by~\citet{D05} for lower bounding active learning, but adjusted to our more general setting. 

%Comparing Theorem~\ref{thm: average splitting lower bound} with Theorems~\ref{thm: 0-1 loss DBAL guarantees} and~\ref{thm: noiseless Bayesian convergence}, we see that {\sc ndbal} is quite close to having an optimal dependence on the average splitting index.

\section{Illustrative examples}
\label{section: examples}

In this section, we look at two specific structure learning settings. The first setting is the problem of learning a ranking over objects with features, where we provide bounds on the average splitting index. Combined with the results from Section~\ref{section: theoretical guarantees}, this gives us bounds on the performance of \textsc{ndbal}.

The second setting is the problem of clustering the real line into $k$ intervals. Here we demonstrate that the choice of structure distance can greatly influence the number of queries needed. In particular, when the structure distance only concerns a constant number of clusters, the label complexity of interactive structure discovery can be far smaller than when a more generic distance depending on the whole structure is used. 

\subsection{Feature-based rankings}

In feature-based ranking, we have distribution $\mu$ over objects, each with corresponding feature vector $x \in \R^d$. A ranking corresponds to a weight vector $w \in \G = S^{d-1}$ (the unit sphere), where $w$ ranks $x$ over $y$ if and only if $\langle w, x \rangle > \langle w, y \rangle$, in which case we write $w(x,y) = 1$ and 0 otherwise.

A natural ranking distance here is the following generalization of the Kendall tau distance:
\[ d_r(w,w') = \pr_{x,y \sim \mu}(w(x,y) \neq w'(x,y)). \]
The following theorem bounds the average splitting index when $\mu$ is spherically symmetric.
\begin{restatable}{thm}{RankingIndex}
\label{thm: ranking index}
Suppose $\mu$ is spherically symmetric. Under distance $d_r(\cdot,\cdot)$, $\G$ has average splitting index $(\frac{1}{16 \lceil \log(2/\epsilon) \rceil}, \epsilon, c \epsilon)$ for some absolute constant $c > 0$.
\end{restatable}

Combining Theorem~\ref{thm: ranking index} with Theorems~\ref{thm: 0-1 loss DBAL guarantees} and~\ref{thm: noiseless Bayesian convergence}, the label complexity of \textsc{ndbal} in this setting grows poly-logarithmically in $1/\epsilon$. % It is not hard to show that the label complexity of using random pairwise queries grows polynomially in $1/\epsilon$, exponentially worse than \textsc{ndbal}. 

\subsection{Clustering on the line} 

Consider the problem of clustering the real line into $k$ intervals where there is some interval $\I$ that we know should be clustered together under the ground truth clustering, and our goal is to identify the other points on the line that should be clustered with $\I$. 

Say there is some measure $\mu$ over the real line, and let $\G_{k,\I}$ denote the set of clusterings of the real line into $\leq k$ intervals such that $\I$ is contained completely in one of these intervals. Note that a clustering $g \in \G_{k,\I}$ can be described by $k-1$ reals $a_1 \leq a_2 \leq \cdots \leq a_{k-1}$. 

The atomic questions consist of pairs of points $(x,y)$, where $g(x,y) = 1$ if they belong to the same cluster and 0 otherwise. A natural distribution $\D$ over atomic questions is the product distribution $\mu \otimes \mu$, and a natural clustering distance is given by
\[ d_c(g,g') = \pr_{x,y \sim \mu}(g(x,y) \neq g'(x,y)). \]
However, if our goal is to identify the cluster that $\I$ belongs to, then a more intuitive clustering distance to use is given by
\[ d_\I(g,g') = \pr_{x \sim \mu}(g(x,\I) \neq g'(x,\I)) \]
where $g(x,\I) = g(x,z)$ for all $z \in \I$. 

Given these two notions of clustering distance, as well as our underlying goal of identifying the cluster that $\I$ belongs to, we ask whether there is a query complexity improvement in using an interactive structure discover algorithm such as \textsc{ndbal} with distance $d_\I(\cdot,\cdot)$ as opposed to just learning with the standard clustering distance $d_c(\cdot,\cdot)$. Informally, we show the following.

\begin{thm}[Informal statement]
\label{thm: clustering separation}
There are settings in which learning under distance $d_c(\cdot,\cdot)$ with any interactive learning algorithm requires exponentially more queries than learning under $d_\I(\cdot,\cdot)$ with \textsc{ndbal}.
\end{thm}

To prove Theorem~\ref{thm: clustering separation}, we derive the following bound on the average splitting index under distance $d_\I(\cdot,\cdot)$.

\begin{restatable}{lemma}{ClusterIdentificationIndex}
\label{lemma: cluster identification index}
Let $\mu(\I) = \alpha$. Under distance $d_\I(\cdot,\cdot)$, $\G_{k,\I}$ has average splitting index $(\frac{1}{16 \lceil \log(2/\epsilon) \rceil}, \epsilon, \frac{\epsilon \alpha}{2})$.
\end{restatable}

\section{Simulations}
\label{section: experiments}

\begin{figure*}
	\begin{center}
		\includegraphics[width=0.75\textwidth]{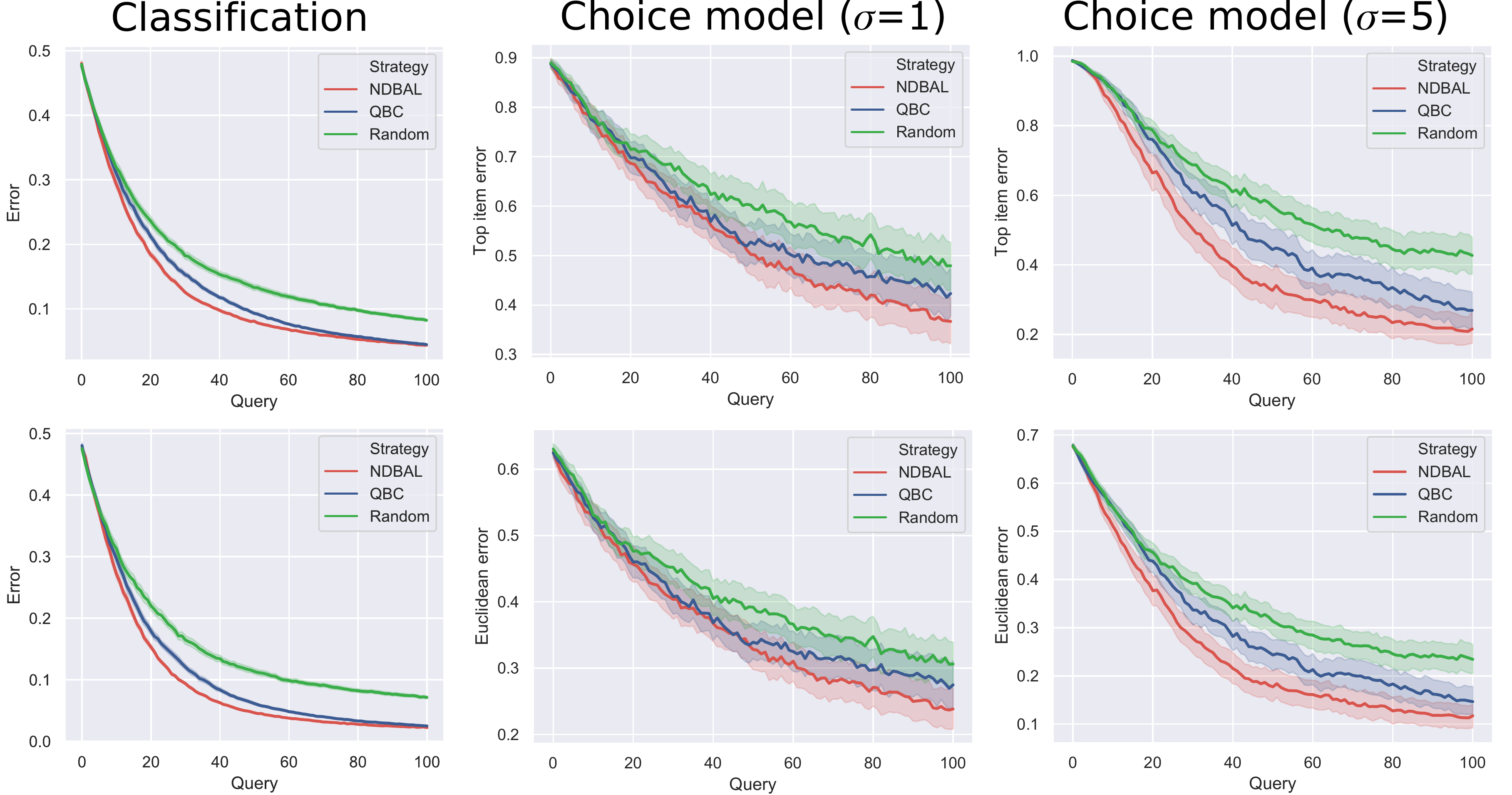}
	\end{center}
	\caption{\emph{Left}: Logistic noise simulations with $d=10$. [\emph{Top to bottom}: $\sigma = $ 5, 10]. \emph{Center and right}: Logit choice model experiments with $\sigma=$ 1, 5. [\emph{Top}: Top-item error. \emph{Bottom}: distance to best item error.]}
	\label{fig: all experiments}
\end{figure*}

We now turn to experimentally evaluating \textsc{ndbal} in two settings: linear classifiers and logit choice models. Before doing so, we discuss a modification to \textsc{ndbal} that allows it to be run in practice.

\paragraph{General-loss \textsc{NDBAL}}
While the posterior update in Equation~\eqref{eqn: 0-1 loss update} enjoys nice theoretical properties, it results in a posterior distribution
that may be intractable to sample from. Thus, we consider a more general update:
\vspace{-0.15em}
\begin{equation}
\label{eqn: general loss update}
\pi_t(g) \ \propto \ \pi_{t-1}(g) \exp(- \beta \ell(g(a_t), y_t))
\end{equation}
where $\ell(\cdot, \cdot)$ is some loss function. When the prior distribution $\pi$ is log-concave, the loss function is convex, and $\G$ is convex, this results in a posterior distribution that is log-concave, and thus efficiently samplable~\citep{LV07}. Moreover, this update was shown to enjoy nice consistency properties for interactive learning strategies that query high variance atoms~\citep{TD18}.

To formalize this setting, let $\Y$ denote the space of answers to atomic questions $\A$, and let $\Z \subset \R^d$ denote some prediction space for structures in $\G$. We view each structure in $\G$ as a function from $\A$ to $\Z$, and we suffer loss $\ell(z,y)$ for predicting $z$ given answer $y$. %Common loss functions include the squared loss where $\Y, \Z \subset \R^d$ and $\ell(z,y) = \|z - y \|^2$, and the logistic loss where $\Y = \{-1,1\}, \Z \subset \R$ and $\ell(z,y) = \ln(1 + e^{-zy})$.

Given this setup, we consider selecting queries $a \in \A$ that approximately minimize
\begin{equation}
\label{eqn: general loss query}
\max_{y \in \Y} \sum_{g, g'} \pi_t(g) \pi_t(g') d(g, g') e^{-\beta( \ell(g(a), y) + \ell(g'(a), y))}.
\end{equation}
When $\ell(\cdot, \cdot)$ is the 0-1 loss and $\beta \rightarrow \infty$, the above corresponds to selecting queries that maximize average splitting. When $\Y$ is finite, we can still use \textsc{select} to choose our query. However, we found that simply drawing a sequence of structure pairs and choosing the query that empirically minimizes equation~\eqref{eqn: general loss query} performed well enough.

\paragraph{Linear classifier simulations}
We consider the problem of learning linear classifiers where the data is distributed uniformly over the unit sphere $\mathcal{S}^{d-1}$. In this setting, there is a target classifier $w^* \in \R^d$, and the goal is to find a vector $w \in \R^d$ minimizing
\[ d(w, w^*) \ = \ \pr_{x\sim \text{unif}(\mathcal{S}^{d-1})}( \sign(\langle w, x \rangle)  \neq   \sign(\langle w^*, x \rangle)) \]
We ran experiments on actively learning such a classifier under the logistic noise model where
$ w^* \sim \N(0, \sigma^2 I_d)$ and $\pr(y \, | \, x, w^*) = \left(1 +e^{- y \langle w^*, x \rangle} \right)^{-1}.$

Figure~\ref{fig: all experiments} shows the performance of \textsc{ndbal} run with the logistic loss against two baselines: random sampling and \textsc{qbc}~\citep{FSST97, TD18}--an active learner that repeatedly samples an atom and two structures and queries the atom if the two structures disagree on it.

\paragraph{Logit choice simulations}
In the logit choice model~\citep{T09}, there is a fixed set of $n$ items, represented as $x_1, \ldots, x_n \in \R^d$, and there is some consumer whose preferences over the items can be captured by a vector $w^* \in \R^d$, such that the consumer prefers item $i$ over item $j$ if and only if $\langle w^*, x_i \rangle > \langle w^*, x_j \rangle$. When presented with a pair of items $(i,j)$, the consumer chooses item $i$ with probability ${1}/({1 + e^{- \langle w^*, x_i - x_j \rangle}})$.

We performed simulations in an interactive setting in which pairs of items are adaptively presented to the consumer. We considered two objectives.
\begin{itemize}
	\item[(i)] Best item identification: identifying $x_{i_{w^*}}$ where $i_{w} = \argmax_i \langle w, x_i \rangle$ is the top item under $w$.
	\item[(ii)] Approximate best item identification: finding an item $j$ such that $\| x_j - x_{i_{w^*}}\|$ is small.
\end{itemize}
We generated $w^* \sim \N(0, \sigma^2 I_d)$ and drew $x_1, \ldots, x_n$ uniformly from $S^{d-1}$. To run \textsc{ndbal}, we used $d(w, w') \ = \ \|x_{i_w} - x_{i_{w'}} \|$ as our  structure distance. The results are displayed in Figure~\ref{fig: all experiments}. 

\paragraph{Experimental summary.} In the appendix, we provide more settings of parameters as well as more information on our experimental setup. Across all our experiments, we found that \textsc{ndbal} generally outperformed \textsc{qbc} and \textsc{random} on the metrics we tested.

% Format:  dim, beta, queries, sigma, ndata

\subsubsection*{Acknowledgements}

The authors thank the anonymous reviewers for suggestions that improved the paper. They also acknowledge the NSF for support under grant CCF-1740833. CT also thanks Wesley Tansey for useful conversations that helped inspire this work.

\bibliography{../references}

\newpage

\appendix
\onecolumn

\section{Experiments continued}

In this section, we discuss our experimental setup more thoroughly and present more results. Each plot depicts $\geq 50$ independent simulations, and the error bands depict 68\% bootstrap confidence intervals. For the \textsc{ndbal} query selection algorithm, we used the heuristic suggested in Section~\ref{section: experiments}: we sampled $m=500$ candidate atoms from $\D$ and $n=300$ pairs of structures from $\pi_t$ and chose the atom that empirically minimized equation~\eqref{eqn: general loss query}. 

\subsection{Models, sampling, and evaluation}

In our experiments, we used the posterior update in equation~\eqref{eqn: general loss update} with $\ell(z,y)$ as the logistic loss, i.e.
\[ \ell(z,y) = \log\left(1 + e^{-zy} \right).\] 
In this setting, it is not possible to express $\pi_t$ in closed form. However, we can still approximately sample from $\pi_t$ using the Metropolis-adjusted Langevin Algorithm (MALA)~\citep{DCWY18}. If we let 
\[ f(w) \ = \ - \sum_{i=1}^t \beta \ell(\langle w, x_i \rangle, y_i) - \frac{1}{2\sigma^2} \|w\|^2 \]
then MALA is a Markov chain in which we maintain a vector $W_t \in \R^d$ and transition to $W_{t+1}$ according to the following process.
\begin{enumerate}
	\item[(i)] Sample $V \sim \N(W_t - \eta \nabla f(W_t), 2 \eta I_d)$.
	\item[(ii)] Calculate $\alpha = \min \left\{1, \exp\left( f(W_t) - f(V) + \frac{1}{4\eta} \left(\|V - W_T + \eta \nabla f(W_t)  \|^2 - \| W_t - V + \eta \nabla f(V) \|^2  \right)  \right) \right\}$.
	\item[(iii)] With probability $\alpha$, $W_{t+1} = V$. Otherwise, set $W_{t+1} = W_t$.
\end{enumerate}
The only hyper-parameter that needs to be set is $\eta > 0$. This parameter should be carefully chosen: if $\eta$ is too large then the walk may never accept the proposed state, and if $\eta$ is too small then the walk may not move far enough to get to a large probability region. The best choice of $\eta$ ultimately depends on the distribution we are sampling from, and unfortunately for us, our distributions are changing. Our fix is to adjust $\eta$ on the fly so that the average number of times that step (iii) rejects is not too close to 0 or to 1. A reasonable rejection rate is about 0.4~\citep{RR98}.

Finally, in all of our evaluations we recorded an approximation of the \emph{average} error of the posterior distribution $\pi_t$. This consists of sampling structures $g_1,\ldots, g_n \sim \pi_t$ and calculating
\[ \widehat{\text{error}}(\pi_t) \ = \ \frac{1}{n} \sum_{i=1}^n d(g_i, g^*) \]
where $d(\cdot, \cdot)$ is the distance function for the task at hand. In our experiments, this distance takes the following forms.
\begin{itemize}
	\item Classification error: $d(w, w') = \pr_{x\sim \text{unif}(\mathcal{S}^{d-1})}( \sign(\langle w, x \rangle)  \neq   \sign(\langle w^*, x \rangle)) =  \frac{1}{\pi} \arccos\left( \frac{\langle w, w' \rangle}{\|w\|\|w'\|} \right)$.
	\item Best item identification: $d(w, w') = \ind[i_{w} \neq i_{w'}]$.
	\item Approximate best item identification: $d(w, w') = \|x_{i_w} - x_{i_{w'}} \|$.
\end{itemize}
In the above, $i_{w} = \argmax_i \langle w, x_i \rangle$ is the top item under $w$ in the choice model setting. We used $n=300$ in our experiments.

\subsection{Classification experiments}
\begin{figure}
	\begin{center}
		\includegraphics[width=.9\textwidth]{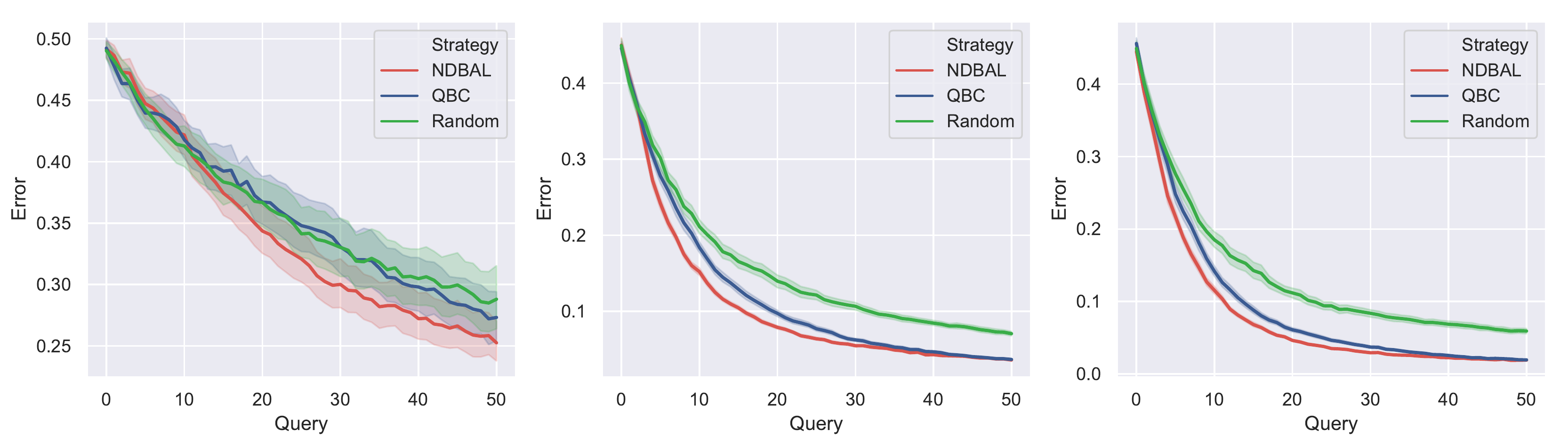} \\
		\includegraphics[width=.9\textwidth]{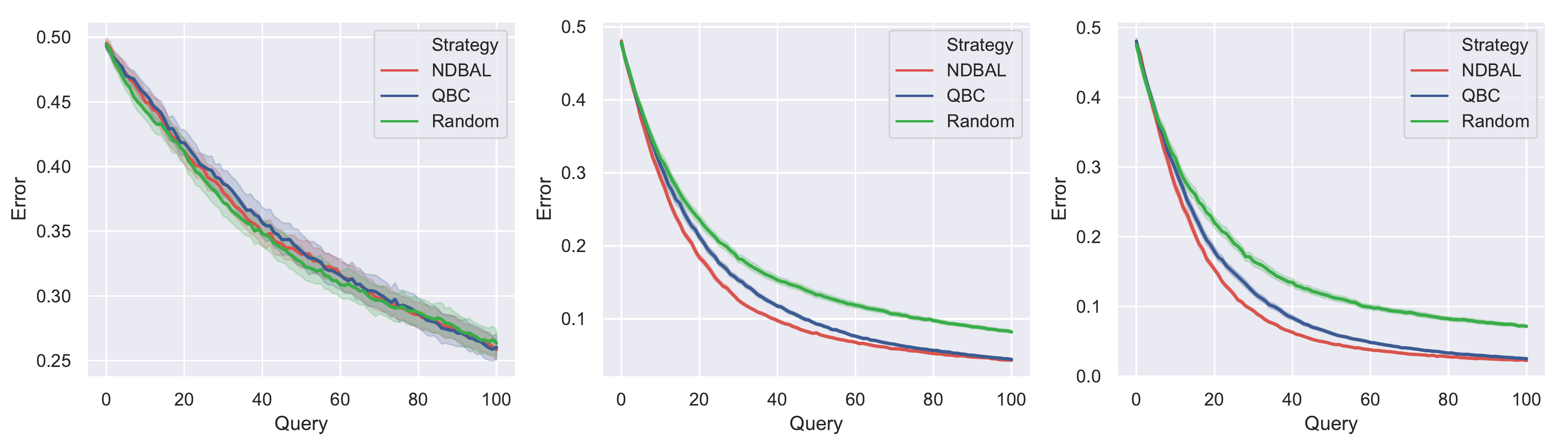} \\
	\end{center}
	\caption{Logistic noise experiments. \emph{Top to bottom}: $d=5, 10$. \emph{Left to right}: $\sigma = 1$, 5, 10.}
	\label{fig: appendix logistic classification experiments}
\end{figure}

In Figure~\ref{fig: appendix logistic classification experiments}, we have classification experiments under logistic noise across different dimensions $d$ and standard deviations $\sigma$. In all of the experiments, we used the logistic loss update on the posterior with $\beta=1$ and a prior distribution of $\N(0, \sigma^2 I_d)$.

\subsection{Logit choice model experiments}
In Figure~\ref{fig: appendix logit choice experiments}, we have logit choice model experiments across different dimensions $d$, numbers of items $n$, and standard deviations $\sigma$. In all of the experiments, we used the logistic loss update on the posterior with $\beta=1$ and a prior distribution of $\N(0, \sigma^2 I_d)$.
 
 % Format:  dim, beta, queries, sigma, ndata
\begin{figure}
	\includegraphics[width=0.5\textwidth]{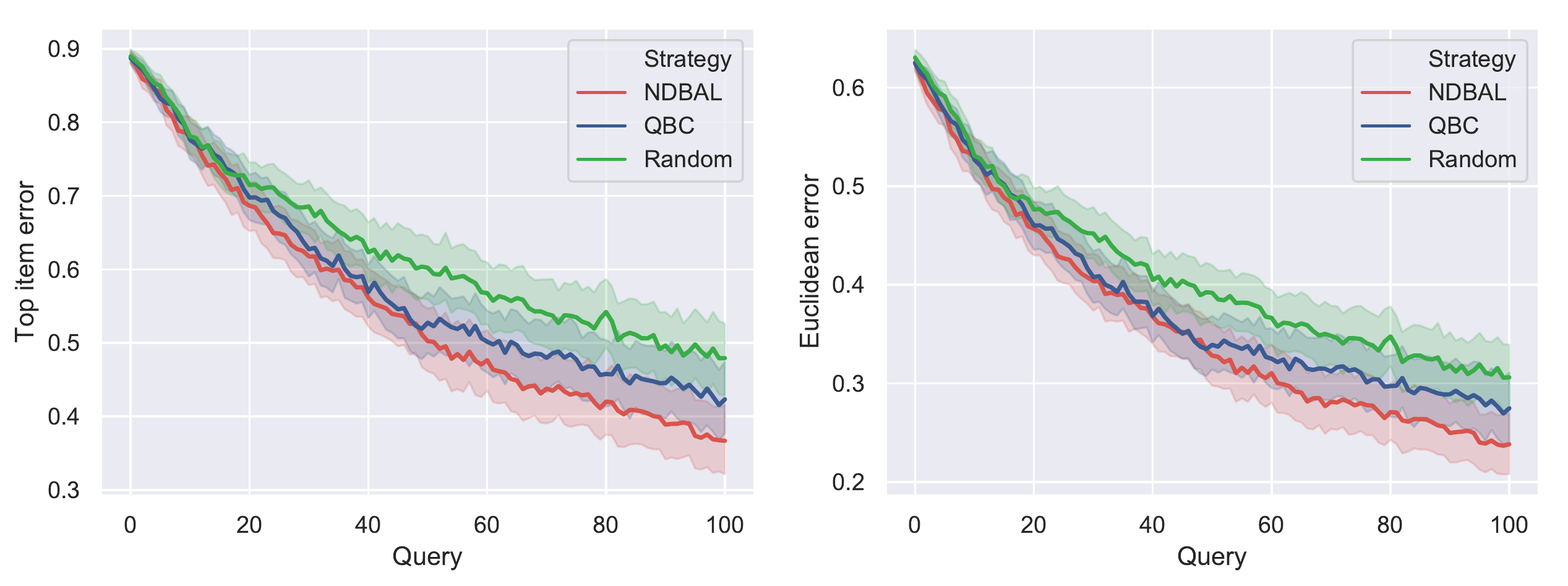} \hspace{1em}
	\includegraphics[width=0.5\textwidth]{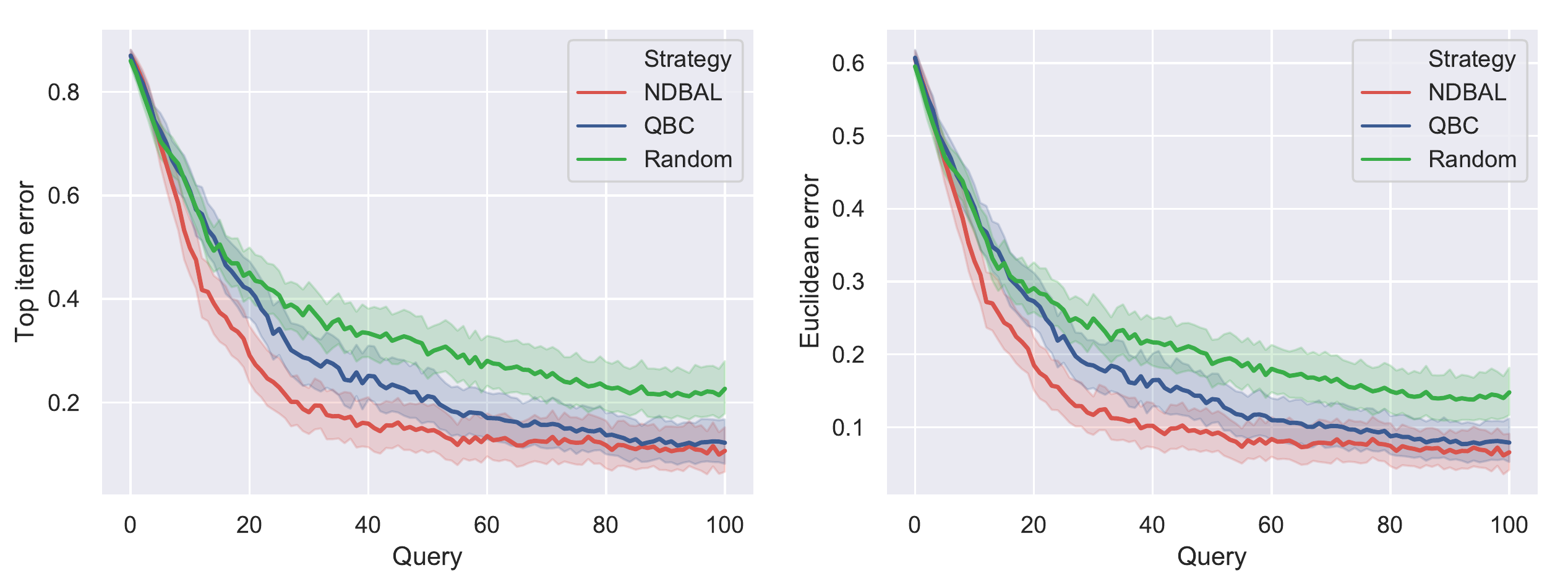} \\
	\includegraphics[width=0.5\textwidth]{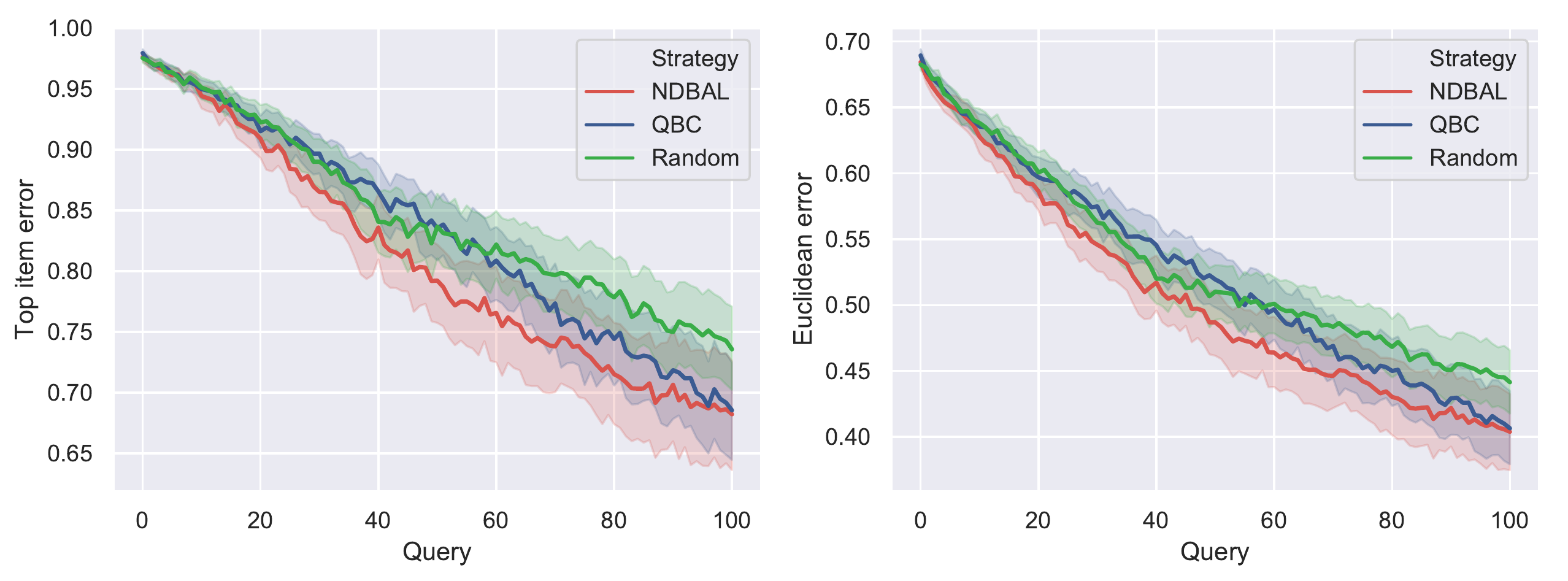} \hspace{1em}
	\includegraphics[width=0.5\textwidth]{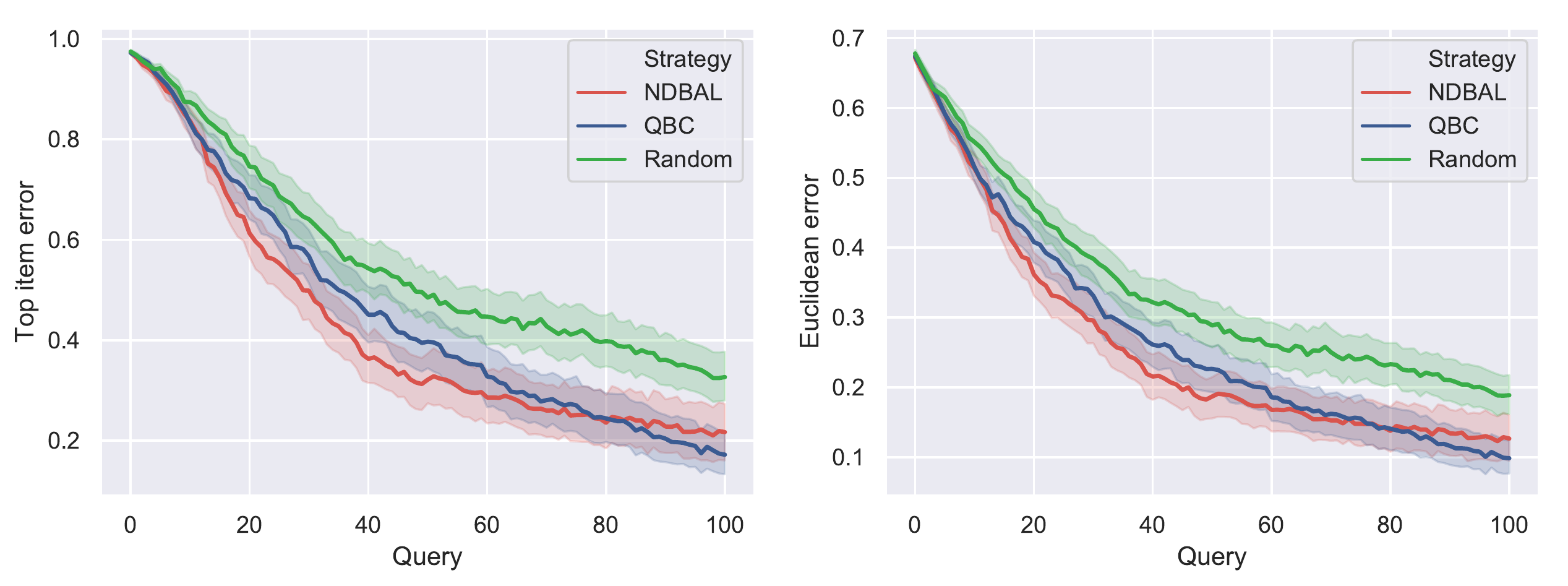} \\
	\includegraphics[width=0.5\textwidth]{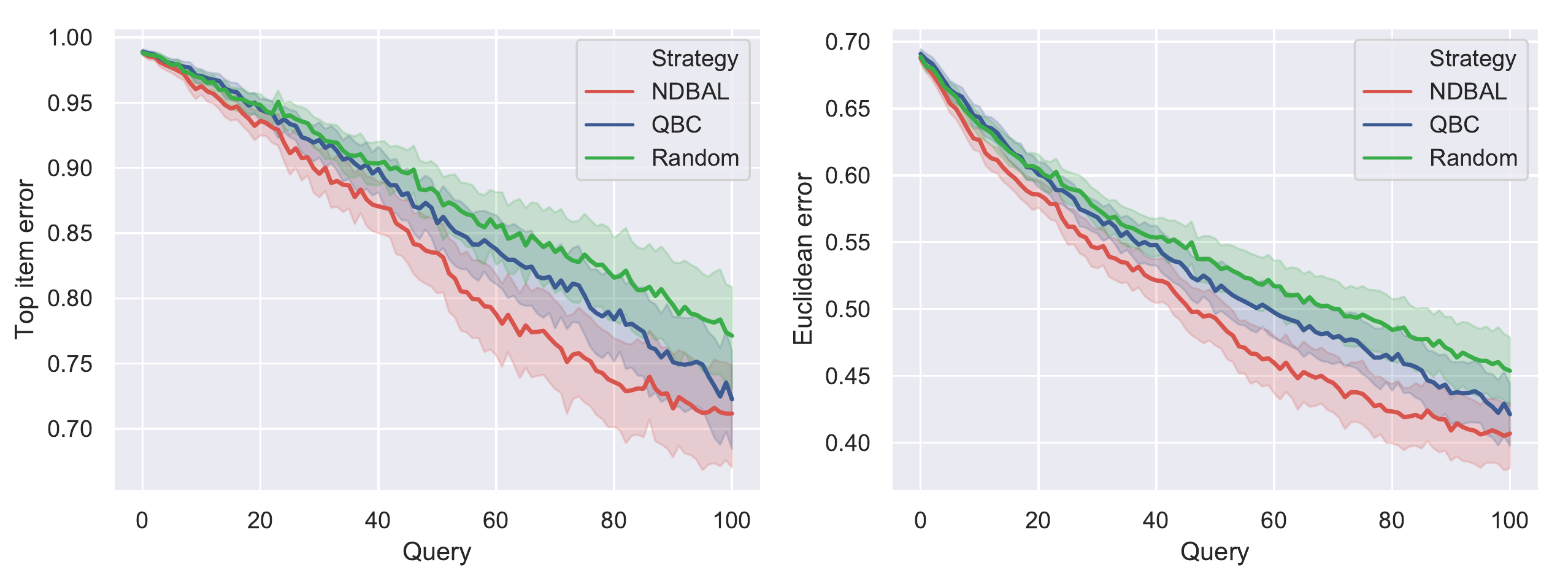} \hspace{1em}
	\includegraphics[width=0.5\textwidth]{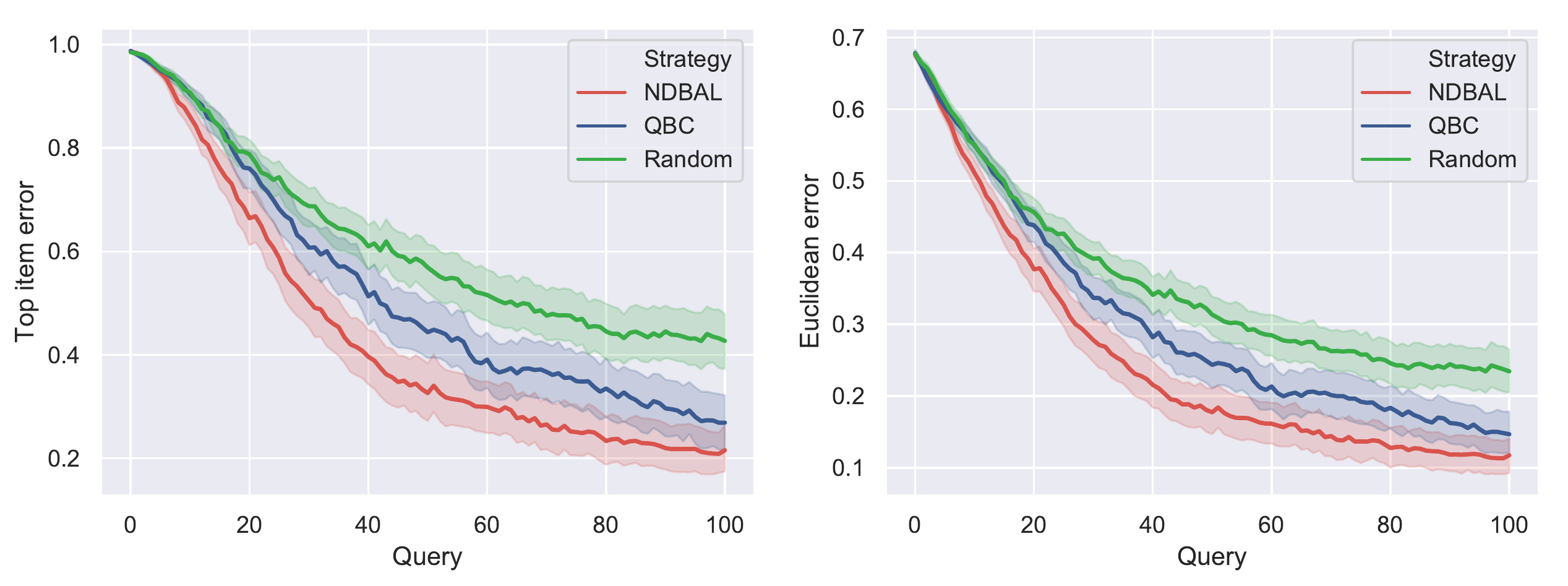}
	\caption{Logit choice model experiments with $d=10$. \emph{Top to bottom}: $n=10, 50, 100$. \emph{Left to right}: $\sigma = 1$, 5.}
	\label{fig: appendix logit choice experiments}
\end{figure}

\section{Dasgupta's splitting index}

We will make use of the original splitting index of \citet{D05} and its multiclass extension by \citet{BH12}. Let $E = ((g_1, g'_1), \ldots, (g_n, g'_n))$ be a sequence of structure pairs. We say that an atom $a$ $\rho$-splits $E$ if
\[ \max_y |E_a^y| \ \leq \ (1-\rho) |E|. \]
$\G$ has splitting index $(\rho, \epsilon, \tau)$ if for any edge sequence $E$ such that $d(g, g') > \epsilon$ for all $(g,g') \in E$, we have
\[ \pr_{a \sim \D}(a \; \rho \text{-splits } E) \geq \tau. \]

The following theorem, which we will use heavily, demonstrates that the average splitting index can be bounded by the splitting index. It is analogous to Lemma~3 of~\citet{TD17}.
\begin{thm}
\label{thm: splitting implies average splitting}
Fix $\G$, $\D$, and $\pi$. If $\G$ has splitting index index $(\rho, \epsilon, \tau)$ then it has average splitting index $(\frac{\rho}{4 \lceil \log_2 1/\epsilon \rceil}, 2\epsilon, \tau  )$.
\end{thm}
From the proof of Lemma~3 by~\citet{TD17}, it is easy to see that so long as $d(\cdot,\cdot)$ is symmetric and takes values in $[0,1]$, the same arguments imply Theorem~\ref{thm: splitting implies average splitting}. 

\section{Proofs from Section~\ref{section: DBAL}}

\subsection{Proof of Lemma~\ref{lem: select lemma}}

To prove Lemma~\ref{lem: select lemma}, we will appeal to the following multiplicative Chernoff-Hoeffding bound~\citep{AV77}.

\begin{lemma}
\label{lem: multiplicative Chernoff bound lemma}
Let $X_1, \ldots, X_n$ be i.i.d. random variables taking values in $[0,1]$ and let $X = \sum X_i$ and $\mu = \E[X]$. Then for $0 < \beta < 1$,
\begin{itemize}
	\item[(i)] $\pr(X \leq (1 - \beta)\mu) \leq \exp\left(	- \frac{\beta^2 \mu}{2} \right)$ and
	\item[(ii)] $\pr(X \geq (1 + \beta)\mu) \leq \exp\left(	- \frac{\beta^2 \mu}{3} \right)$.
\end{itemize}
\end{lemma}

The key observation in proving Lemma~\ref{lem: select lemma} is that if $a$ $\rho$-average splits $\pi$, then for all $y \in \Y$ we have
\[ \avg(\pi) - \pi(\G_a^y)^2\avg(\pi|_{\G_a^y}) \ \geq \ \rho \, \avg(\pi). \]
On the other hand, if $a$ does not  $\rho$-average split $\pi$, then there is some $y \in \Y$ such that
\[ \avg(\pi) - \pi(\G_a^y)^2\avg(\pi|_{\G_a^y}) \ < \ \rho \, \avg(\pi). \]
Moreover, if $g, g' \sim \pi$, then 
\[ \E[d(g,g')(1 - \ind[g(a) = y = h'(a)])] \ = \ \avg(\pi) - \pi(\G_a^y)^2\avg(\pi|_{\G_a^y}). \]
Using these facts, along with Lemma~\ref{lem: multiplicative Chernoff bound lemma}, we have the following result.

\SelectLemma*
\begin{proof}
Define $K_N^{a,y} = \inf \{ K \, : \,  S_K^{a,y} \geq N \}$. Recalling that $S^{a,y}_k = \sum_{i=1}^k d(g_i, g'_i) (1 - \ind[g_i(a) = y = g'_i(a)])$, we have the following relationship between $K_N^{a,y}$ and $S^{a,y}_k$.
\begin{align*}
\pr(K^{a,y}_N \leq k) \ &= \ \pr(S^{a,y}_{k_o} \geq N \text{ for some } k_o \leq k) \ \leq \  \pr(S^{a,y}_{k} \geq N) \\
\pr(K^{a,y}_N > k) \ &= \ \pr(S^{a,y}_{k_o} < N \text{ for all } k_o \leq k) \ = \ \pr(S^{a,y}_{k} < N)
\end{align*}

Now let $a^*$ be the atom that $\rho$-average splits $\pi$. Then for all $y \in \Y$, we have
\[ \pr\left(K_N^{a^*, y} > \frac{N}{(1+\epsilon/2)(1-\epsilon)  \rho \, \avg(\pi)} \right) \ \leq \ \exp\left(- \frac{N \epsilon^2(1+\epsilon)^2}{8( 1 - \epsilon(1+\epsilon)/2)}  \right). \]
On the other hand we know for any data point $a$ that does not $(1-\epsilon)\rho$-average split $\pi$, there is some $y \in \Y$ such that
\[ \pr\left(K_N^{a, y} \leq \frac{N}{(1+\epsilon/2)(1-\epsilon)  \rho \, \avg(\pi)} \right) \ \leq \ \exp\left(- \frac{N \epsilon^2}{12( 1 - \epsilon/2)}  \right). \]
Taking a union bound over $\Y$ and all the $a$'s, we have 
\[ \pr\left(\text{we choose } a_i \text{ that does not } (1-\epsilon)\rho \text{-average split } \pi \right) \ \leq \ | \Y | \exp\left(-\frac{N\epsilon^2}{4(2-\epsilon)} \right) + m \exp\left(-\frac{N \epsilon^2}{6(2+\epsilon)}   \right). \]
By our choice of $N$, this is less than $\delta$.
\end{proof}

\section{Proofs from Section~\ref{section: theoretical guarantees}}

\subsection{Proof of Lemma~\ref{lem: general k 0-1 loss decrease}}
\GeneralKLoss*
\begin{proof}
To simplify notation, take $\pi = \pi_{t-1}$. Suppose that we query $a \in \A$. Enumerate the potential responses as $\Y = \{ y_1, y_2, \ldots, y_m \}$. The definition of average splitting implies that there exists a symmetric matrix $R \in [0,1]^{m \times m}$ satisfying 
\begin{itemize}
	\item $R_{ii} \leq 1 - \rho$ for all $i$,
	\item $\sum_{i,j} R_{ij} = 1$, and
	\item $R_{ij} \, \avg(\pi) =  \sum_{g \in \G_a^{y_i}, g' \in \G_a^{y_j}} \pi(g) \pi(g') d(g,g')$.
\end{itemize}
Let us assume w.l.o.g. that $g^*(a) = y_1$. Define the quantity 
\[ Q_a^{i} \ := \ \pi(G_a^{y_i}) + e^{-\beta} \sum_{j \neq i} \pi(G_a^{y_j}) \ = \  \pi(G_a^{y_i}) + e^{-\beta} (1 - \pi(G_a^{y_i})) \ \leq \ 1. \]

We now derive the form of $\avg(\pi_t)$. In the event that $y_t = i$, we have 
\begin{align*}
\avg(\pi_t) &= \sum_{h,h' \in \H} \pi_t(h) \pi_t(h') d(h, h') \\ 
&= \left(\frac{1}{Q_a^i} \right)^2  \left( \sum_{g, g' \in \G_a^{y_i}} \pi(g) \pi(g') d(g, g') + 2 e^{-\beta} \sum_{j\neq i}  \sum_{g \in \G_a^{y_1}, g' \in \G_a^{y_j}} \pi(g) \pi(g') d(g, g') \right. \\
& \hspace{6em} \left. + e^{-2\beta} \sum_{j \neq i,  k \neq i} \sum_{g \in \G_a^{y_j}, g' \in \G_a^{y_k}} \pi(g) \pi(g') d(g, g') \right) \\ \displaybreak[3]
&= \left(\frac{1}{Q_a^i} \right)^2 \left(R_{ii} + 2e^{-\beta} \sum_{j \neq i} R_{ij} + e^{-2\beta} \sum_{j\neq i, k \neq i} R_{jk} \right) \avg(\pi) \\
&=  \left(\frac{1}{Q_a^i} \right)^2 \left(R_{ii} + 2e^{-\beta} \sum_{j \neq i} R_{ij} + e^{-2\beta} \left(1 - R_{ii} - 2 \sum_{j \neq i} R_{ij} \right) \right) \avg(\pi) \\
&= \left(\frac{1}{Q_a^i} \right)^2 \left(e^{-2\beta} + (1-e^{-2\beta}) R_{ii} + 2(e^{-\beta} - e^{-2\beta}) \sum_{j \neq i} R_{ij} \right) \avg(\pi) .
\end{align*}
We can also derive the form of $\frac{1}{\pi_t(g^*)^k}$:
\[ \frac{1}{\pi_t(g^*)^k} 
\ = \ 
\begin{cases}
\left(\frac{Q_a^1}{\pi(g^*)}\right)^k &\text{ if } y_t = y_1 \\
 \left(\frac{Q_a^i}{e^{-\beta} \pi(g^*)}\right)^k &\text{ if } y_t = y_i \neq y_1
\end{cases} 
 \]
Define \[ \Delta_t \ := \ \frac{\pi(g^*)^k}{\avg(\pi)} \cdot \E\left[ \frac{\avg(\pi_{t})}{\pi_{t}(g^*)^k} \right] . \]
If we take $\eta(y_i | a) = \gamma_i$ and assume w.l.o.g. that $\gamma_1 > \gamma_2 \geq \gamma_3 \geq \cdots$, then
\begin{align*}
\Delta_t &= \gamma_1 (Q_a^1)^{k-2} \left(e^{-2\beta} + (1-e^{-2\beta}) R_{11} + 2(e^{-\beta} - e^{-2\beta}) \sum_{j \neq 1} R_{1j} \right) \\
& \hspace{3em }+ \sum_{i \geq 2} \gamma_i (Q_a^1)^{k-2} e^{k\beta} \left(e^{-2\beta} + (1-e^{-2\beta}) R_{ii} + 2(e^{-\beta} - e^{-2\beta}) \sum_{j \neq i} R_{ij} \right) \\
&\leq  (1-\gamma_1)e^{(k-2)\beta}  + \gamma_1\left(e^{-2\beta} + (1-e^{-2\beta}) R_{11} + 2(e^{-\beta} - e^{-2\beta}) \sum_{j \neq 1} R_{1j} \right) \\
& \hspace{3em }+\gamma_2 \left((e^{k\beta}-e^{(k-2)\beta}) \sum_{i\geq 2} R_{ii} + 2(e^{(k-1)\beta} - e^{(k-2)\beta}) \sum_{i\geq 2} \sum_{j \neq i} R_{ij} \right) \\
&\leq  (1-\gamma_1)e^{(k-2)\beta} + \gamma_1(1-e^{-2\beta}) R_{11} + \gamma_2 (e^{k\beta}-e^{(k-2)\beta}) \sum_{i\geq 2} R_{ii}\\
& \hspace{3em } + \left( \gamma_1(e^{-\beta} - e^{-2\beta}) + \gamma_2(e^{(k-1)\beta}-e^{(k-2)\beta}) \right)\left( 1 - \sum_{i\geq 1}R_{ii} \right)
\end{align*}
Using the inequalities $1+x \leq e^x \leq 1 + x + x^2$ for $|x| \leq 1$ and Assumption~\ref{assump: bounded noise}, we can verify that the following inequalities hold for our choice of $\beta$:
\begin{gather*}
\gamma_2  (e^{k\beta}-e^{(k-2)\beta}) \ \leq \ \gamma_1(e^{-\beta} - e^{-2\beta}) + \gamma_2(e^{(k-1)\beta}-e^{(k-2)\beta}) \ \leq \ \gamma_1(1-e^{-2\beta}) \label{eqn: inequality 1}\\
(1-\gamma_1)e^{(k-2)\beta}  + \gamma_1(1-e^{-2\beta}) \ \leq \ 1 \label{eqn: inequality 2} \\
\gamma_1(1 - e^{-\beta}) + \gamma_2(e^{(k-1)\beta}-e^{(k-2)\beta}) \ \leq \ -\beta \lambda/2 \label{eqn: inequality 3}
\end{gather*}
Using our restrictions on the structure of $R$, the above inequalities imply
\begin{align*}
\Delta_t \ &\leq \ (1-\gamma_1)e^{(k-2)\beta} + (1- \rho)\gamma_1(1-e^{-2\beta}) + \rho \left( \gamma_1(e^{-\beta} - e^{-2\beta}) + \gamma_2(e^{(k-1)\beta}-e^{(k-2)\beta}) \right) \\
\ &= \  (1-\gamma_1)e^{(k-2)\beta}  + \gamma_1(1-e^{-2\beta}) + \rho  \left( \gamma_1(1 - e^{-\beta}) + \gamma_2(e^{(k-1)\beta}-e^{(k-2)\beta}) \right) \\
\ &\leq \ 1 + \rho  \left( \gamma_1(1 - e^{-\beta}) + \gamma_2(e^{(k-1)\beta}-e^{(k-2)\beta}) \right) \\
\ &\leq \ 1 - \rho \lambda \beta/2 . \qedhere
\end{align*}
\end{proof}

\subsection{Proof of Lemma~\ref{lem: pi-mass supermartingale}}
\PiMassSupermartingale*
\begin{proof}
Suppose we query $a$ at step $t$. Denote by $\gamma_i = \eta(y_i \, | \, a)$ and $\pi_i = \pi_{t-1}(\G_a^{y_i})$, and assume w.l.o.g that $g^*(a) = y_1$ and $\gamma_1 > \gamma_2 \geq \gamma_3 \geq \cdots$. Then we have
\begin{align*}
\E\left[ \frac{1}{\pi_t(g^*)^k} \ | \ \pi_{t-1}(g^*) \right] &= \frac{\gamma_1(\pi_1 + e^{-\beta}(1-\pi_1))^k}{\pi_{t-1}(g^*)^k} + \sum_{i \geq 2} \frac{\gamma_i (e^{\beta} \pi_i + 1-\pi_i)^k}{\pi_{t-1}(g^*)^k} \\
&= \frac{1}{\pi_{t-1}(g^*)^k} \left(\gamma_1(\pi_1 + e^{-\beta}(1-\pi_1))^k + \sum_{i\geq 2} \gamma_i  (e^{\beta} \pi_i + 1-\pi_i)^k  \right) 
\end{align*}
Denote the term in parenthesis by $\Delta_t$. Using the inequalities $1+x \leq e^x \leq 1 + x + x^2$ for $|x| \leq 1$, for our choice of $\beta$ we have
\begin{align*}
\Delta_t \ &\leq \ \gamma_1(\pi_1 + (1-\beta + \beta^2)(1-\pi_1))^k + \sum_{i\geq 2} \gamma_i  ((1+\beta + \beta^2) \pi_i + 1-\pi_i)^k \\
\ &= \ \gamma_1(1 - \beta(1-\beta)(1-\pi_1))^k + \sum_{i\geq 2} \gamma_i  (1 + \pi_i\beta(1+\beta))^k \\
\ &\leq \ \gamma_1 \exp(- k \beta(1-\beta)(1-\pi_1)) + \sum_{i\geq 2} \gamma_i  \exp(k \pi_i\beta(1+\beta)) \\
\ &\leq \  \gamma_1 (1 - k \beta(1-\beta)(1-\pi_1) + (k \beta(1-\beta)(1-\pi_1))^2) + \sum_{i\geq 2} \gamma_i  (1 + k \pi_i\beta(1+\beta)+ (k \pi_i\beta(1+\beta))^2) \\
\ &= \ 1 + k \beta \left( (1+\beta)\sum_{i \geq 2} \gamma_i \pi_i - \gamma_1 (1-\beta)(1-\pi_1)  \right) + k^2 \beta^2 \left(  (1+\beta)^2\sum_{i \geq 2} \gamma_i \pi_i^2 + \gamma_1 (1-\beta)^2(1-\pi_1)^2  \right) \\
\ &\leq \ 1 + k \beta (1-\pi_1) \left( \gamma_2 (1+\beta) - \gamma_1 (1-\beta) \right) + k^2 \beta^2 (1-\pi_1)^2 \left( \gamma_2 (1+\beta)^2 + \gamma_1 (1-\beta)^2  \right) \\
\ &= \ 1 + k \beta (1-\pi_1) \left( \beta \left( \gamma_1 + \gamma_2 \right)\left(1 + k(1 - \pi_1) + \beta^2 k(1-\pi_1)  \right) - (\gamma_1 - \gamma_2) (1 + 2\beta^2 k (1-\pi_1)   \right) \\
\ &\leq \ 1 + k \beta (1-\pi_1) \left( \beta k  - \lambda  \right) \ \leq \ 1. \qedhere
\end{align*}
\end{proof}

\subsection{Proof of Lemma~\ref{lem: expected splitting lower bound}}
Recall our definitions of the splitting index. Let $E = ((g_1, g'_1), \ldots, (g_n, g'_n))$ be a sequence of structure pairs. We say that an atom $a$ $\rho$-splits $E$ if
\[ \max_y |E_a^y| \ \leq \ (1-\rho) |E|. \]
$\G$ has splitting index $(\rho, \epsilon, \tau)$ if for any edge sequence $E$ such that $d(g, g') > \epsilon$ for all $(g,g') \in E$, we have
\[ \pr_{a \sim \D}(a \; \rho \text{-splits } E) \geq \tau. \]

\begin{lemma}
\label{lem: coarse splitting bounds}
Pick $\gamma, \epsilon > 0$. If $\G$ is finite and Assumption~\ref{assump: identifiable distance} holds, then there exists a constant $p > 0$ such that $\G$ has splitting index $((1-\gamma)p, \epsilon, \gamma p)$
\end{lemma}
\begin{proof}
Given Assumption~\ref{assump: identifiable distance} and the finiteness of $\G$, we know that there is some $p >0$ such that for any $g,g' \in \G$ satisfying $d(g,g') > 0$, we have $ \pr_{a \sim \D}(g(a) \neq g'(a)) \geq p $. Now suppose that we have a collection of edges $E \subset {\G \choose 2}$ such that $d(g,g') > \epsilon$ for all $(g,g') \in E$. A random atom $a \sim \D$ will split some random number $Z$ of these edges. Note that $\E Z \geq p |E|$. Moreover, by Markov's inequality, we have
\[ \pr(Z \geq (1-\gamma)p |E|) |E| \ \geq \ \E Z - (1-\gamma)p |E| \ \geq \ p |E| - (1-\gamma)p |E| \ = \ \gamma p |E|. \]
Simplifying the above, and substituting our definition of splitting gives us
\[ \pr_{a \sim \D}(a \, (1-\gamma)p\text{-splits } E) \ \geq \  \gamma p. \qedhere \]
\end{proof}

Lemma~\ref{lem: coarse splitting bounds} and Theorem~\ref{thm: splitting implies average splitting} together imply the following corollary.

\begin{cor}
\label{cor: average splitting index lower bound}
If $\G$ is finite and Assumption~\ref{assump: identifiable distance} holds, then there exists a constant $p > 0$ such that $\G$ has average splitting index $\left(\frac{p}{8(\log_2(1/\epsilon) + 2)}, \epsilon, p/2  \right)$.
\end{cor}

Given this result, we can now prove the following claim.

\ExpectedSplittingLowerBound*
\begin{proof}
By Corollary~\ref{cor: average splitting index lower bound}, there is some constant $p > 0$ such that every distribution $\pi_t$ is $(\rho, \tau)$-average splittable with
\[ \rho \ := \ \frac{p}{8\left(\log_2 \frac{1}{\avg(\pi_t)} + 2 \right)} \; \;  \text{ and } \; \; \tau \ := \ p/2. \]
Suppose that \textsc{ndbal} draws $m_t \geq 1$ candidate queries at round $t$. By the definition of average splittability, we have 
\[ \pr(\text{at least one of } m_t \text{ draws } \rho\text{-average splits } \pi_{t-1}) \ \geq \ 1 - (1- \tau)^{m_t} \ \geq \ \tau \ \geq \ p/2. \]
Conditioned on both of this happening, Lemma~\ref{lem: select lemma} tells us that {\sc select} will choose a point that $(1-\alpha)\rho$-average splits $\pi_t$ with probability $1-\delta$. Putting these together, along with the fact that $\rho_t \geq 0$ always, gives us the lemma.
\end{proof}

\subsection{Proof of Theorem~\ref{thm: DBAL consistency}}

\DBALConsistency*
\begin{proof}
Let $X_t = \avg(\pi_{t})$ and $Y_t = 1/\pi_t(g^*)^2$. Since $\beta \leq \lambda/10$, Lemmas~\ref{lem: general k 0-1 loss decrease} and~\ref{lem: expected splitting lower bound}, together with the inequality $x/(1+ \log(1/x)) \geq x^2$ for $x \in (0,1)$, imply
\begin{equation}
\label{eqn: supermartingale inequality}
\E[X_t Y_t \, | \, \F_{t-1}] \ \leq \ X_{t-1}Y_{t-1} - c X_{t-1}^2 Y_{t-1}
\end{equation}
for some constant $c > 0$. Since $X_t Y_t$ and $Y_t$ are positive supermartingales, we have that $X_t Y_t \rightarrow Z$ and $Y_t \rightarrow Y$ for some random variables $Z$, $Y$ almost surely. Moreover, since $Y_t, Y \geq 1$ almost surely, we have $X_t^2 Y_t \rightarrow W$ for some random variable $W$ almost surely.

Iterating expectations in equation~\eqref{eqn: supermartingale inequality} and using the fact that $X_t Y_t \geq 0$, we have
\[ 0 \ \leq \ \E[X_t Y_t] \ \leq \ \frac{\avg(\pi_o)}{\pi_o(g^*)^2} - c \sum_{i=1}^{t-1} \E[X_i^2 Y_i]. \]
In particular, we know $\lim_{t \rightarrow \infty} \E[X_t^2 Y_t] = 0$. By Fatou's lemma, this implies
\[ 0 \ \leq \ \E \left[\lim_{t \rightarrow \infty} X_t^2 Y_t \right]  \ \leq \  \lim_{t \rightarrow \infty} \E[X_t^2 Y_t] \ = \ 0. \]
Thus, we have
\[ \lim_{t\rightarrow \infty} \frac{\avg(\pi_t)^2}{\pi_t(g^*)^2} \ = \  \lim_{t\rightarrow \infty}  X_t^2 Y_t \ = \ 0 \]
almost surely. By the Continuous Mapping Theorem, this implies $\frac{\avg(\pi_{t})}{\pi_{t}(g^*)} \rightarrow 0$ almost surely. The inequality \[ 0 \leq \E_{g\sim\pi_t}[d(g,g^*)] \leq \frac{\avg(\pi_{t})}{\pi_{t}(g^*)} \]
finishes the proof.
\end{proof}

\subsection{Proof of Theorem~\ref{thm: 0-1 loss DBAL guarantees}}

\DBALGuarantees*
\begin{proof}
We will show that for some round $t$, \textsc{ndbal} must encounter a posterior distribution $\pi_t$ satisfying $\avg(\pi_t)/\pi(g^*)^2 \leq \epsilon$ while using the resources described in the theorem statement. By Lemma~\ref{lem: average diameter guarantee}, this will imply that $\E_{g \sim \pi_t}[d(g,g^*)] \leq \epsilon$ for the same round $t$.

Lemma~\ref{lem: pi-mass supermartingale} implies that $1/\pi_t(g^*)^2$ is a positive supermartingale for our choice of $\beta$. From standard martingale theory~\citep{R13}, we have $\pi_t(g^*)^2 \geq \delta \pi(g^*)^2/4$ for $t=1,\ldots, T$ with probability at least $1 - \delta/4$.

Conditioned on this event, we have by a union bound that if we sample $m_t = \frac{1}{\tau} \log \frac{4t(t+1)}{\delta}$ data points at every round $t$, then with probability $1- \delta/4$, one of those data points will $\rho$-average split $\pi_{t}$ for every round in which $\avg(\pi_t)/\pi_t(g^*)^2 > \epsilon$. Conditioned on drawing such points, Lemma~\ref{lem: select lemma} tells us that for all rounds $t$, {\sc select} terminates with a data point that $\rho/2$-average splits $\pi_t$ with probability $1-\delta/4$ after drawing $n_t$ hypotheses, for the value of $n_t$ given in the statement.

Let us condition on all of these events happening. For round $t$ define the random variable 
\[ \Delta_t = 1 - \frac{\avg(\pi_t)}{\pi_t(g^*)^2}\cdot\frac{\pi_{t-1}(g^*)^2}{\avg(\pi_{t-1})}. \] 
If $\pi_{t-1}$ satisfies $\avg(\pi_t)/\pi_t(g^*)^2 > \epsilon$, then the query $x_t$ $\rho/2$-average splits $\pi_{t-1}$. By Lemma~\ref{lem: general k 0-1 loss decrease},
\begin{align*}
\E[\Delta_t \, | \, \F_{t-1}]  \ \geq \ \frac{1}{2} \rho \lambda \beta (1-\beta).
\end{align*}
Now suppose by contradiction that $\avg(\pi_t)/\pi_t(g^*)^2 > \epsilon$ for $t=1,\ldots, T$. Then we have $\E[\Delta_1 + \ldots + \Delta_T] \geq \frac{T}{2}\rho \lambda \beta (1-\beta)$. To see that this sum is concentrated about its expectation, we notice that $\Delta_t \in [1 - e^{2\beta}, 1]$ since 
\[ e^{-\beta} \pi_{t-1}(g) \ \leq \ \pi_t(g) \ \leq \ e^\beta \pi_{t-1}(g) \]
for all $g \in \G$ which implies
\[ e^{-2\beta} \ \leq \ \frac{\avg(\pi_t)}{\pi_t(g^*)^2}\cdot\frac{\pi_{t-1}(g^*)^2}{\avg(\pi_{t-1})} \ \leq \ e^{2\beta} . \]
By the Azuma-Hoeffding inequality~\citep{A67, H63}, if $T$ achieves the value in the theorem statement, then with probability $1-\delta$,
\begin{align*}
\Delta_1 + \cdots + \Delta_T \ > \ \frac{1}{2} \E[\Delta_1 + \cdots + \Delta_T]  \ \geq \ \frac{T}{8}\rho \lambda \beta (1-\beta) \ \geq \ \ln \frac{1}{\epsilon \pi(g^*)^2}.
\end{align*}
However, this is a contradiction since
\begin{align*} 
\epsilon \ < \ \frac{\avg(\pi_T)}{\pi_T(g^*)^2}  \ = \ (1- \Delta_1)\cdots(1 - \Delta_T)\frac{\avg(\pi)}{\pi(g^*)^2} \ \leq \ \exp\left(-(\Delta_1 + \cdots + \Delta_T)\right) \frac{1}{\pi(g^*)^2}. 
\end{align*}
Thus, with probability $1-\delta$, we must have encountered a distribution $\pi_t$ in some round $t=1,\ldots, T$ satisfying $\avg(\pi_t)/\pi_t(g^*)^2 \leq \epsilon$.
\end{proof}

\subsection{Proof of Theorem~\ref{thm: noiseless Bayesian convergence}}

To begin, we will utilize the following result on our stopping criterion.

\begin{lemma}
\label{lem: stopping criterion}
Pick $\epsilon, \delta > 0$ and let $n_t = \frac{48}{\epsilon}\log \frac{t(t+1)}{\delta}$. If at the beginning of each round $t$, we draw $E = (\{g_1, g'_1\}, \ldots, \{g_{n_t}, g'_{n_t} \}) \sim \pi_t$, then with probability $1-\delta$
\begin{align*}
\frac{1}{n_t} \sum_{i=1}^{n_t} d(g_i, g'_i)  > \frac{3\epsilon}{4}  \ &\text{ if } \ \avg(\pi_t) > \epsilon \\
\frac{1}{n_t} \sum_{i=1}^{n_t} d(g_i, g'_i) \leq   \frac{3\epsilon}{4} \ &\text{ if } \ \avg(\pi_t) \leq \epsilon/2
\end{align*}
for all rounds $t \geq 1$.
\end{lemma}
The proof of Lemma~\ref{lem: stopping criterion} follows from applying a union bound to Lemma~7 of~\citet{TD17}.

For a round $t$, let $V_t$ denote the version space, i.e. the set of structures consistent with the responses seen so far. Then we may write
\[ \pi_t(g) \ = \ \frac{\pi(g) \ind[g \in V_t]}{\pi(V_t)} \ \  \text{ and } \ \ \nu_t(g) \ = \ \frac{\nu(g) \ind[g \in V_t]}{\nu(V_t)} . \]
Assumption~\ref{assump: Bayesian} tells us that we have the following upper bound.
\[ D(\pi_t, \nu_t)  \ \leq \ \lambda^2 \avg(\pi_t) .\]
Thus, the average diameter of $\avg(\pi_t)$ is a meaningful surrogate for the objective $D(\pi_t, \nu_t)$ in this setting. 

Recalling the definition of average splitting, we know that if we always query points that $\rho$-average the current posterior, then after $t$ rounds we will have
\[ \pi(V_t)^2 \avg(\pi_t) \ \leq \ (1-\rho)^t \pi(V_0)^2 \avg(\pi) \ \leq \ e^{-\rho t}.  \]
While this demonstrates that the potential function $\pi(V_t)^2 \avg(\pi_t)$ is decreasing exponentially quickly, it does not by itself guarantee that $\avg(\pi_t)$ is itself decreasing. What is needed is a lower bound on the factor $\pi(V_t)$. The following lemma, which is a generalization of a result due to~\citet{FSST97}, provides us with just that, provided that $\G$ has bounded {graph dimension}.

\begin{restatable}{lemma}{BayesianLowerBound}
\label{lem: Bayesian posterior lower bound}
Suppose $g^* \sim \nu$ where $\nu$ is a prior distribution over a hypothesis class $\G$ with graph dimension $d_G$, and say $|\Y| \leq k$. Let $c >0$ and $a_1, \ldots, a_m$ be any atomic questions, and let $V^* = \{ g \in \G \, : \, g(a_i) = g^*(a_i) \text{ for all } i \}$, then
\[ \pr\left(\log\left( \frac{1}{\nu(V^*)} \right) \geq  c +  d_G \log \frac{em(k+1)}{d_G}\right) \ \leq \ e^{-c}. \]
\end{restatable}

To prove this, we need the following generalization of Sauer's lemma.
\begin{lemma}[Corollary~3~\citep{HL95}]
\label{lem: generalization of Sauer's lemma}
Let $d,m,k$ be s.t. $d \leq m$. Let $F \subset \{ 1, \ldots, k \}^m$ s.t. $F$ has graph dimension less than $d$. Then,
\[ |F| \leq \sum_{i=0}^d {m \choose i} (k+1)^i \leq \left( \frac{em(k+1)}{d} \right)^d. \]
\end{lemma}

\begin{proof}[Proof of Lemma~\ref{lem: Bayesian posterior lower bound}]
Let $V_1, \ldots, V_N \subset \G$ denote the partition of $\G$ induced by our atomic questions. Note that if $g^* \sim \nu$, then the probability $V^* = V_i$ is exactly $\nu(V_i)$. Let $S \subset \{1, \ldots N\}$ consist of all indices $i$ satisfying $\log \frac{1}{\nu(V_i)} \geq c + \log N$. Rearranging, we have
\[ \sum_{i \in S} \nu(V_i) \ \leq \ e^{-c} \cdot \frac{|S|}{N} \ \leq \ e^{-c}. \]
From Lemma~\ref{lem: generalization of Sauer's lemma}, we have $\log N \leq d_G \log \frac{em(k+1)}{d_G}$, which finishes the proof.
\end{proof}

Given the above, we are now ready to prove Theorem~\ref{thm: noiseless Bayesian convergence}.

\NoiselessBayesianConvergence*
\begin{proof}
If we use the stopping criterion from Lemma~\ref{lem: stopping criterion} with the threshold $3\epsilon/4\lambda^2$, then at the expense of drawing an extra $\frac{48\lambda^2}{\epsilon} \log \frac{t(t+1)}{\delta}$ hypotheses for each round $t$, we are guaranteed that with probability $1 - \delta$ if we ever encounter a round $t$ in which $\avg(\pi_t) \leq \epsilon/(2\lambda^2)$ then we terminate and we also never terminate whenever $\avg(\pi_K) > \epsilon$. Thus if we do ever terminate at some round $t$, then with high probability
\[ D(\pi_t, \nu_t) \ \leq \ \lambda^2 \avg(\pi_t) \ \leq \ \epsilon . \]

It remains to be shown that we will encounter such a posterior. Note that if we draw $m_t \geq \frac{1}{\tau} \log \frac{t(t+1)}{\delta}$ atoms per round, then with probability $1-\delta$ one of them will $\rho$-average split $\pi_t$ if $\avg(\pi_t) > \epsilon/(2\lambda^2)$. Conditioned on this happening, Lemma~\ref{lem: select lemma} guarantees that that with probability $1-\delta$ {\sc select} finds a point that $\rho/2$-average splits $\pi_t$ while drawing at most $O\left(\frac{\lambda^2}{\epsilon \rho} \log \frac{(m_t + |\Y|)t(t+1)}{\delta}\right)$. 

If after $T$ rounds we still have not terminated, then $\avg(\pi_T) > \epsilon/(2\lambda^2)$. However, we also know
\[ \pi(V_T)^2 \avg(\pi_T) \ \leq \ e^{-\rho T/2}. \]

Now suppose that in each round $t$, we have seen $m_t$ atoms $x^{(t)}_1, \ldots, x^{(t)}_{m_t}$, and define 
\[ V_{T^*} = \{ h \in \HH \, : \, h(x^{(t)}_{i}) =  h^*(x^{(t)}_{i}) \text{ for } t=1,\ldots, T, i = 1, \ldots, m_t \}. \]  
Clearly, $V_{T^*} \subset V_T$. By Lemma~\ref{lem: Bayesian posterior lower bound}, we have with probability $1-\delta$,
\[ \pi(V_T) \ \geq \ \pi(V_{T^*}) \ \geq \ \frac{1}{\lambda} \nu(V_{T^*}) \ \geq \ \frac{1}{\lambda} \cdot \frac{\delta}{T(T+1)} \left( \frac{d_G}{em^{(T)}(|\Y|+1)} \right)^{d_G} \]
for all rounds $T \geq 1$, where $m^{(T)} = \sum_{t=1}^T m_t$.

Plugging this in with the above, we have
\[ \avg(\pi_T) \ \leq \ \frac{e^{-\rho T/2}}{\pi(V_T)^2} \ \leq \ \lambda^2 \exp\left( 2d_G \log \frac{em^{(T)}(|\Y|+1)}{d_G} + 2 \log \frac{T(T+1)}{\delta} - \frac{\rho T}{2} \right). \]
Suppose $m_t = \frac{1}{\tau} \log \frac{t(t+1)}{\delta}$. Then we can upper bound $m^{(T)}$ as
\[ m^{(T)}
\ = \ \sum_{t=1}^T m_t 
\ \leq \ \frac{T}{\tau} \log \frac{T(T+1)}{\delta} .
\]
Putting everything together, we have 
\[
\frac{\epsilon}{2\lambda^2} \ \leq \ \avg(\pi_T)
 \ \leq \ \lambda^2 \exp\left(2 \log \frac{T(T+1)}{\delta} + 2d_G \log\left( \frac{e(|\Y|+1)}{d_G} \cdot \frac{T}{\tau}\log \frac{T(T+1)}{\delta}  \right)  - \frac{\rho T}{2}  \right). \]
Letting $C = 2d_G \log \frac{e(|\Y|+1)}{d_G \tau}$ and $b = \frac{1}{\delta}$, the right-hand side is less than $\epsilon/(2\lambda^2)$, whenever
\[ T \ \geq \ \frac{2}{\rho} \max \left\{ C + \log \frac{2\lambda^4}{\epsilon} + 6(d_G+1) \log T, C + \log \frac{2\lambda^4}{\epsilon} + \log b + 2d_G \log\left(3b \log(b) \right) \right\} .\]
Additionally, note that $T \geq \frac{2}{\rho}\left( C + \log \frac{1}{\epsilon} + 6(d_G+1) \log T \right)$, whenever
\[ T \ \geq \  \frac{4}{\rho} \max \left\{ C + \log \frac{2\lambda^4}{\epsilon},  24(d_G+1) \log^2\left( \frac{96(d_G + 1)}{\rho} \right)  \right\} . \]
The value of $T$ provided in the theorem statement, satisfies all of these inequalities. Thus, with probability $1-4\delta$, we must have encountered a round in which $\avg(\pi_t) < \epsilon/(2\lambda^2)$ and terminated.
\end{proof}

\subsection{Proof of Theorem~\ref{thm: average splitting lower bound}}

The following result is analogous to Theorem~2 of~\citet{D05}.
\begin{thm}
\label{thm: splitting lower bound}
Fix $\G$ and $\D$. Suppose that $\G$ does not have splitting index $(\rho, \epsilon, \tau)$ for some $\rho,\epsilon \in (0,1)$ and $\tau \in (0,1/2)$. Then any interactive learning strategy which with probability $> 3/4$ over the random sampling from $\D$ finds a structure $g \in \G$ within distance $\epsilon/2$ of any target in $\G$ must draw at least $1/\tau$ atoms from $\D$ or must make at least $1/\rho$ queries.
\end{thm}
From the proof of Theorem~2 of~\citet{D05}, it is easy to see that so long as $d(\cdot,\cdot)$ is symmetric, the same arguments imply Theorem~\ref{thm: splitting lower bound}. For completeness, we include its proof here.
\begin{proof}
Since $\G$ does not have splitting index $(\rho, \epsilon, \tau)$, there is some set of edges $E \subset {\G \choose 2}$ such that $d(g, g') > \epsilon$ for all $(g,g') \in E$ and 
\[ \pr_{a \sim \D}(a \; \rho \text{-splits } E) < \tau. \]
Let $V$ denote the vertices of $E$. Then distinguishing between structures in $V$ requires at least $1/\rho$ queries or at least $1/\tau$ atoms.

To see this, suppose we draw less than $1/\tau$ atoms. Then with probability at least $(1-\tau)^{1/\tau} \geq 1/4$ none of these atoms $\rho$-splits $E$, i.e. for each of these atoms there is some response $y \in \Y$ such that less than $\rho |E|$ edges are eliminated. Thus, there is some $g^* \in V$ such that requires us to query at least $1/\rho$ atoms to distinguish it from the rest of the structures in $V$.
\end{proof}

Combining the above with Theorem~\ref{thm: splitting implies average splitting}, we have the following corollary.

\AverageSplittingLowerBound*

\section{Proofs from Section~\ref{section: examples}}
\subsection{Proof of Theorem~\ref{thm: ranking index}}

We will utilize the following result from~\citet{D05}.

\begin{lemma}[Lemma~11~\citep{D05}]
\label{lem: subspace lemma}
For any $d \geq 2$, let $x, y$ be vectors in $\R^d$ separated by an angle of $\theta \in [0, \pi]$. Let $\tilde{x}, \tilde{y}$ be their
projections into a randomly chosen two-dimensional subspace. There is an absolute constant $c_o >0$ (which
does not depend on $d$) such that with probability at least 3/4 over the choice of subspace, the angle between
$\tilde{x}$ and $\tilde{y}$ is at least $c_o \theta$.
\end{lemma}

Given the above, we prove Theorem~\ref{thm: ranking index}.

\RankingIndex*

The proof of Theorem~\ref{thm: ranking index} closely mirrors that of Theorem~10~\citep{D05}. For completeness, we produce its proof here.

\begin{proof}
We make two key observations here.
\begin{itemize}
	\item A weight vector $w \in \G$ ranks $x$ over $y$ if and only if $\langle w, x-y \rangle > 0$.  
	\item If $x,y$ are drawn from a spherically symmetric distribution, then $z = x-y$ also follows a spherically symmetric distribution.
\end{itemize}
From these two observations, we know that if $w,w' \in \G$, then $d(w,w') = \theta/\pi$ where $\theta$ is the angle lying between $w$ and $w'$. 

Suppose $w_1,w'_1, \ldots, w_n, w'_n$ are a sequence of edges such that $d(w_i,w'_i) \geq \epsilon$, which implies their corresponding angles satisfy $\theta_i \geq \epsilon \pi$. Suppose we project the pairs onto a randomly drawn 2-d subspace, to get $\tilde{w}_1,\tilde{w}'_1, \ldots, \tilde{w}_n, \tilde{w}'_n$. Let $c_o$ be the absolute constant from Lemma~\ref{lem: subspace lemma}. Call an edge $\tilde{w}_i,\tilde{w}'_i$ good if the resulting angle satisfies $\tilde{\theta}_i \geq c_o \epsilon \pi$. 

By Lemma~\ref{lem: subspace lemma}, the expected number of good edges for a randomly chosen 2-d subspace is $n/2$. By Markov's inequality, with probability $1/2$, at least $n/2$ edges are good.

Let us suppose that we have chosen a 2-d subspace/plane that results in at least $n/2$ good edges. Call these projected edges $\tilde{w}_1,\tilde{w}'_1, \ldots, \tilde{w}_m, \tilde{w}'_m$. Without loss of generality, assume that the clockwise angle $\tilde{\theta}_i$ from $\tilde{w}_i$ to $\tilde{w}'_i$ satisfies $c_o \epsilon \pi \geq \tilde{\theta}_i \leq \pi$. Notice that if $z_o$ is in our plane and satisfies $\langle \tilde{w}_i, z_o \rangle \geq 0$ for at least $n/2$ edges and $\langle \tilde{w}'_i, z_o \rangle \leq 0$ for at least $n/2$ edges, then querying any points $x_o, y_o$ such that $x_o - y_o = z_o$ will eliminate at least half of the $\tilde{w}_i$. Moreover, it is enough to query any pair $x,y$ such that $x-y = z$ satisfies that $x$'s counterclockwise angle is in the range $[0, c_o \epsilon \pi]$ or $[\pi, \pi + c_o \epsilon \pi]$, since such a pair will eliminate either $\tilde{w}_i$ or $\tilde{w}'_i$. Thus, querying such an $x,y$ pair will result in eliminating at least $1/2$ of the good edges, which is at least $1/4$ of all the edges.
 
Since $z = x-y$ follows a spherically symmetric distribution, the probability of drawing such a pair is at least $c_o \epsilon \pi/2.$ Thus, the splitting index here is $(1/4, \epsilon, c_o \epsilon \pi/2)$, and Theorem~\ref{thm: ranking index} follows by applying Theorem~\ref{thm: splitting implies average splitting}.
\end{proof}

\subsection{Proof of Lemma~\ref{lemma: cluster identification index}}
\ClusterIdentificationIndex*
\begin{proof}
We will first bound the splitting index and then invoke Theorem~\ref{thm: splitting implies average splitting}. Suppose that $g_1, g'_1, \ldots, g_n, g'_n \in \G_{k, \alpha}$ are a sequence of edges satisfying $d_\I(g_i, g'_i) \geq \epsilon$ for all $i=1,\ldots, n$. Note that for each $g_i, g'_i$ there are associated reals $\ell_i < u_i$ and $\ell'_i < u'_i$ such that \[ \ell_i, \ell'_i \ \leq \I \ \leq \ u_i, u'_i. \]
From the definition of $d_\I(g_i, g'_i)$, we have
\[ \epsilon \ \leq \ d_\I(g_i, g'_i) \ = \ \mu(\ell_i, \ell'_i ) + \mu(u_i, u'_i )    \]
where $\mu(a,b)$ is the probability mass of the interval bounded by $a$ and $b$. Call an edge \emph{left-leaning} if $\mu(\ell_i, \ell'_i ) \geq \epsilon/2$ and \emph{right-leaning} if $\mu(u_i, u'_i ) \geq \epsilon/2$. 

Suppose without loss of generality that at least half of the edges are right-leaning (the case where half are left-leaning can be handled symmetrically), and order them as $g_1, g'_1, \ldots, g_m, g'_m$ such that $u_1 \leq u_2 \leq \cdots \leq u_m$. Moreover, let us also assume without loss of generality that $u_i < u'_i$. Let $r$ denote the point $u_i < r \leq u'_i$ such that $\mu(u_i, r) = \epsilon/2$. Suppose we query a pair $x,y$ where $x \in \I$ and $y \in (u_{m/2}, r)$, notice that such a pair satisfies.
\[ x < u_1 \leq \cdots \leq u_{m/2} < y < u'_{m/2} \leq \cdots \leq u'_m. \]
If we query this pair and the result is that they should belong to the same cluster, then we may eliminate at least one endpoint of edges $g_1, g'_1, \ldots, g_{m/2}, g'_{m/2}$. On the other hand, if the result is that they should belong to different clusters, then we may eliminate at least one endpoint of edges $g_{m/2}, g'_{m/2}, \ldots, g_{m}, g'_{m}$. In either case, we eliminate at least half of these $m$ edges. Since this is only the right-leaning edges, at least one quarter of the original edges are eliminated. Finally, the probability of drawing such a pair $x,y$ is $\alpha \cdot \epsilon$. 

Thus, $\G_{k,\I}$ has splitting index $(1/4, \epsilon, \alpha \epsilon)$. Theorem~\ref{thm: splitting implies average splitting} finishes the proof.
\end{proof}

\subsection{Proof of Theorem~\ref{thm: clustering separation}}

\begin{figure}
	\begin{center}
		\includegraphics[width=.3\textwidth]{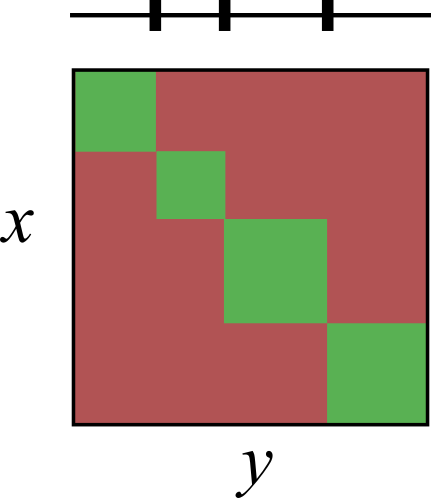}
	\end{center}
	\caption{Viewing an interval-based clustering as a classifier over $\R^2$. The green regions correspond to `must-link' constraints, and the red regions correspond to `cannot-link' constraints.}
	\label{fig: cluster unrolling}
\end{figure}

We will make use of the following result from~\citet{D05}.

\begin{lemma}[Corollary~3~\citep{D05}]
\label{lem: star-shaped bound}
Suppose there are structures $g_o, g_1, \ldots, g_N \in \G$ such that
\begin{enumerate}
	\item $d(g_o, g_i) > \epsilon$ for all $i=1, \ldots, N$ and
	\item the sets $\{a \, : \, g_o(a) \neq g_i(a) \}$ are disjoint for all $i=1, \ldots, N$.
\end{enumerate}
Then for any $\tau > 0$ and any $\rho > 1/N$, $\G$ is not $(\rho, \epsilon, \tau)$-splittable. Thus, any active learning scheme that finds $g \in \G$ satisfying $d(g,g^*) < \epsilon/2$ for any $g^* \in \G$ must use at least $N$ labels in the worst case.
\end{lemma}

Given this, we have the following lemma lower bounding the query complexity of a particular subset of $\G_{k,\I}$.

\begin{restatable}{lemma}{ClusteringLowerBounds}
\label{lem: clustering lower bounds}
Say $\mu(\I) \leq 1/2$. There is a subset $\G_o \subset \G_{k+2,\I}$ of $N = \min \{ k, \frac{1}{\sqrt{8\epsilon}} \} + 1$ clusterings such that learning $\G_o$ under distance $d_c(\cdot,\cdot)$ requires at least $N-1$ queries, no matter how many unlabeled data points are drawn.
\end{restatable}
\begin{proof}
For ease of exposition, say that $\mu$ is uniform over the interval $[0,1]$ and that $\I = [0, \alpha]$ for some $\alpha \leq 1/2$. We will consider the case where $k \leq \frac{1}{\sqrt{8\epsilon}}$, the other case can be proven symmetrically.

Define $g_o$ as the clustering with dividing points \[a_1 = \alpha, \, a_2 = \alpha + \frac{1-\alpha}{k}, \, a_3 = \alpha + \frac{2(1-\alpha)}{k}, \, \ldots, \, a_{k} = \alpha + \frac{(k-1)(1-\alpha)}{k} . \]
We also define $g_i$ as the clustering with the same dividing points except it has an additional dividing point at $b_i = \frac{a_i + a_{i+1}}{2} = \alpha + \frac{(2i - 1)(1-\alpha)}{2k}$ for $i = 1, \ldots k$, where we take $a_{k+1} = 1$. Then it can be seen that
\[ d(g_o, g_i) \ = \ 2 \cdot \pr_{x \sim \mu}( x \in (a_i, b_i)) \cdot \pr_{y \sim \mu}( y \in (b_i, a_{i+1})) \ = \ \frac{1}{2}\left( \frac{1- \alpha}{k} \right)^2 \ \geq \ \epsilon .  \]
Moreover, we also have that the sets $\{(x,y) \, : \, g_o(x,y) \neq g_i(x,y) \}$ are disjoint for all $i=1, \ldots, N$. This is readily observed after making the transformation from an interval-based clustering to binary classifier over $[0,1]^2$. Applying Lemma~\ref{lem: star-shaped bound} finishes the proof.
\end{proof}

Given Lemmas~\ref{lemma: cluster identification index} and~\ref{lem: clustering lower bounds}, we can now prove Theorem~\ref{thm: clustering separation}.

%\ClusteringSeparation*
\begin{mythm}{\ref{thm: clustering separation}}{ \textit{(Formal statement) }}
\textit{Let $\epsilon >0$. There is a setting of $k = \Theta(1/\sqrt{\epsilon})$ and a subset $\G \subseteq \G_{k+2,\I}$ that is polynomially-sized in $k$ such that any active learning algorithm that is guaranteed to find any target in $\G$ up to distance $\epsilon$ in distance $d_c(\cdot, \cdot)$ must make at least $\Omega(k)$ queries, but \textsc{ndbal} with distance $d_\I(\cdot, \cdot)$ and prior $\pi$ uniform over $\G$ requires $O(\log^2 (k/\epsilon \delta))$ queries.}
\end{mythm}
\begin{proof}
Take $k = \Theta(1/\sqrt{\epsilon})$ and let $\G_o \subset \G_{k+2,\I}$ be the subset from Lemma~\ref{lem: clustering lower bounds}. Take $\G$ to be any subset of $\G_{k+2, \I}$ such that (a) $\G$ has size polynomial in $k$ and (b) $\G_o \subseteq \G$. By Lemma~\ref{lem: clustering lower bounds}, we know that learning under distance $d_c(\cdot, \cdot)$ requires at least $|\G_o| = \Theta(k)$ queries. 

On the other hand, consider running \textsc{ndbal} with distance $d_\I(\cdot, \cdot)$ and prior $\pi$ uniform over $\G$. The results in Theorem~\ref{thm: 0-1 loss DBAL guarantees} and Lemma~\ref{lemma: cluster identification index} tell us that \textsc{ndbal} requires $O(\log^2	(k/\epsilon))$ queries to find a posterior $\pi_t$ over $\G$ such that $\E_{g \sim \pi_t}[d_\I(g,g^*)] \leq \epsilon$. To turn this into a high probability result, simply apply Markov's inequality to get that \textsc{ndbal} requires $O(\log^2 (k/\epsilon \delta))$ queries in order to find a posterior $\pi_t$ such that with probability $1-\delta$ if $g \sim \pi_t$ then $d_\I(g,g^*) \leq \epsilon$.
\end{proof}

\section{Noisy fast convergence}

In this section, we give rates of convergence in the Bayesian setting under noise. We start by defining the quantity
\[ Z_t = \sum_{g \in \G} \pi(g) \exp\left(-\beta	 \sum_{i=1}^t \ind[g(x_i) \neq y_i] \right) .\]
The following lemma is analogous to Lemma~\ref{lem: general k 0-1 loss decrease}.

\begin{lemma}
\label{lem: Bayesian potential decrease}
Pick $\beta, \rho > 0$. If at step $t$, our query $\rho$-average splits $\pi_{t-1}$, then
\[ Z_t^2 \Phi(\pi_t) \ \leq \ \left[1- \rho(1 - e^{-\beta})  \right] Z_{t-1}^2 \Phi(\pi_{t-1}).    \]
\end{lemma}
\begin{proof}
Suppose that we query atom $a_t$ and receive label $y_t$. Enumerate the potential responses as $\Y = \{ y_1, y_2, \ldots, y_m \}$. The definition of average splitting implies that there exists a symmetric matrix $R \in [0,1]^{m \times m}$ satisfying 
\begin{itemize}
	\item $R_{ii} \leq 1 - \rho$ for all $i$,
	\item $\sum_{i,j} R_{ij} = 1$, and
	\item $R_{ij} \, \avg(\pi) =  \sum_{g \in \G_a^{y_i}, g' \in \G_a^{y_j}} \pi(g) \pi(g') d(g,g')$.
\end{itemize}
Define the quantity
\[ Q_a^{i} \ := \ \pi(G_a^{y_i}) + e^{-\beta} \sum_{j \neq i} \pi(G_a^{y_j}) \ = \  \pi(G_a^{y_i}) + e^{-\beta} (1 - \pi(G_a^{y_i})) \ \leq \ 1. \]
Note that if $y_t = y_i$, we have
\begin{align*}
Q_a^i 
\ = \ \sum_{g} \pi_{t-1}(g) \exp\left(-\beta \ind[g(a_t) \neq y_t] \right) 
\ = \  \sum_{g} \frac{1}{Z_{t-1}} \pi(g) \exp\left(-\beta	 \sum_{j=1}^t \ind[g(a_j) \neq y_j] \right) 
\ = \ \frac{Z_t}{Z_{t-1}}
\end{align*}
Thus, if we observe $y_t = y_i$, then
\begin{align*}
Z_t^2 \avg(\pi_{t}) &= (Q_a^i Z_{t-1})^2 \sum_{g,g'} \frac{1}{(Q_a^i)^2} \pi_{t-1}(g) \pi_{t-1}(g') d(g,g') \exp\left(-\beta (\ind[g(a_t) \neq y_i] + \ind[g(a_t) \neq y_t]) \right)\\
&= \left( R_{ii} + e^{-2\beta} \sum_{j,k \neq i}R_{jk} + e^{-\beta} \cdot 2 \sum_{j \neq i} R_{ij} \right) Z_{t-1}^2 \avg(\pi_{t-1}) \\
&\leq  \left((1-\rho) + e^{-\beta} \rho  \right) Z_{t-1}^2 \avg(\pi_{t-1}) \ = \ \left( 1 - \rho(1-e^{-\beta}) \right) Z_{t-1}^2 \avg(\pi_{t-1}). \qedhere
\end{align*}
\end{proof}

Suppose we receive query/label pairs $(a_1, y_1), \ldots, (a_t, y_t)$ where the noise level at $a_i$ is $q_i$, then the true posterior distribution under Assumption~\ref{assump: Bayesian} is
\[ \nu_t(g) \ = \  \frac{1}{\widehat{Z}_t}\nu(g) \exp\left( - \sum_{i=1}^t\ind[g(a_i) \neq y_i] \ln \frac{1-q_i}{q_i} \right) \]
where $\widehat{Z}_t$ is the normalizing constant
\[ \widehat{Z}_t \ = \ \sum_{g} \nu(g) \exp\left( - \sum_{i=1}^t \ind[g(a_i) \neq y_i)] \ln \frac{1-q_i}{q_i} \right). \]
The following lemma will be useful in bounding this quantity.
\begin{lemma}
\label{lem: concentration log-odds}
Suppose $Y_1, \ldots, Y_t$ are independent random variables such that
\[ Y_i \ = \ 
\begin{cases}
\ln \frac{1-q_i}{q_i} & \text{ with probability } q_i \\
0 & \text{ with probability } 1-q_i
\end{cases} \]
With probability $1-\delta$, we have
\[ \sum_{i=1}^t Y_i \ \leq \ \sum_{i=1}^t q_i \ln \frac{1-q_i}{q_i} + \sqrt{t \ln \frac{2}{\delta}} \left(\ln \frac{2t}{\delta}\right). \qedhere   \]
\end{lemma}
\begin{proof}
We begin by partitioning the random variables $Y_i$ into two groups. We say $Y_i$ is `small' if $q_i \leq \frac{\delta}{2t}$ and 'big' otherwise. Then with probability at least $1-\delta/2$, all small $Y_i$ satisfy $Y_i = 0$. Let us condition on this happening. 

\medskip

Now each big $Y_i$ takes values in $[0, \ln \frac{2t}{\delta}]$. By Hoeffding's inequality, we have that with probability at least $1-\delta/2$
\[ \sum_{i=1}^t Y_i \ \leq \ \sum_{i=1}^t \E[Y_i] + \sqrt{t \ln \frac{2}{\delta}} \left(\ln \frac{2t}{\delta}\right)\ \leq \  \sum_{i=1}^t q_i \ln \frac{1-q_i}{q_i} + \sqrt{t \ln \frac{2}{\delta}} \left(\ln \frac{2t}{\delta}\right). \qedhere \]
\end{proof}

Given the above, we can lower bound $\widehat{Z}_t$ under Assumption~\ref{assump: Bayesian}.

\begin{lemma}
\label{lem: normalizing constant lower bounds}
Let $\delta \in (0,1)$ and let $\G$ have graph dimension $d_G$. Suppose Assumption~\ref{assump: Bayesian} holds. If in the course of running \textsc{ndbal} we observe $m$ atoms, of which we query $a_1,\ldots, a_t$ where the noise level at $a_i$ is $q_i$, then with probability $1-\delta$ over the randomness of the responses we observe,
\begin{align*}
\log \frac{1}{\widehat{Z}_t} \ &\leq \  \log \frac{2}{\delta} + d_G \log \frac{em(|\Y| + 1)}{d_G} + \sum_{i=1}^t q_i \ln \frac{1-q_i}{q_t} + \sqrt{t \log \frac{3}{\delta}}\left(\log \frac{3t}{\delta} \right)
\end{align*}
\end{lemma}
\begin{proof}
By Assumption~\ref{assump: Bayesian}, we know $g^* \sim \nu$. Let $U$ be the set of $m$ atoms observed in running \textsc{ndbal} and let $V^* = \{ g \in \G \, : \, g(a) = g^*(a) \text{ for } a \in U \}$. By Lemma~\ref{lem: Bayesian posterior lower bound}, we have with probability $1-\delta/2$
\[ \log \frac{1}{\nu(V^*)} \ \leq \ \log \frac{2}{\delta} + d_G \log \frac{em(|\Y| + 1)}{d_G} . \]
Now let $g \in V^*$ and say the responses on atoms $a_1,\ldots, a_t$ are $y_1,\ldots, y_t$, respectively. By Lemma~\ref{lem: concentration log-odds}, we have with probability $1-\delta/2$
\[ \sum_{i=1}^t \ind[g(a_i) \neq y_i] \ln \frac{1-q_i}{q_i} \ \leq \ \sum_{i=1}^t q_i \ln \frac{1-q_i}{q_t} + \sqrt{t \log \frac{6}{\delta}}\left(\log \frac{6t}{\delta}  \right). \]
Combining the above concentration results with the inequality
\begin{align*}
\widehat{Z}_t \ &\geq \ \sum_{g \in V^*} \nu(g) \exp \left(- \sum_{i=1}^t \ind[g(a_i) \neq y_i] \ln \frac{1-q_i}{q_i} \right)
\end{align*}
gives us the lemma.
\end{proof}

We will assume that the noise distribution is restricted to classification noise.

\begin{assump}
\label{assump: classification noise}
There exists a $q \in (0,1)$ and $g^* \in \G$ such that $\eta(g^*(a) \, | \, a) = 1 - q$.
\end{assump}

If we know the noise level, then the appropriate setting of $\beta$ is $\ln \frac{1-q}{q}$, in which case we recover the bound
\begin{equation}
\label{eqn: diameter bounds posterior distance}
\D(\pi_t, \nu_t) \ \leq \ \lambda^2 \avg(\pi_t).
\end{equation}

Given the above, we can now prove the following theorem.

\begin{thm}
Suppose $\G$ has average splitting index $(\rho, \epsilon/(2\lambda^2), \tau)$ and graph dimension $d_G$. If Assumptions~\ref{assump: Bayesian} and~\ref{assump: classification noise} hold, $\gamma = \frac{\rho}{2}\cdot\frac{1-2q}{1-q} - q \ln \frac{1-q}{q} > 0$, and $\beta = \ln \frac{1-q}{q}$, then with probability $1-\delta$ modified {\sc ndbal} terminates with a distribution $\pi_t$ satisfying $D(\pi_t, \nu_t) \leq \epsilon$ while using the following resources:
\begin{itemize}
	\item[(a)] less than $T = {O}\left(\frac{1}{\gamma} \log^3 \frac{1}{\gamma \delta} + \frac{d_G}{\gamma} \log \left( \frac{ d_G \lambda |\Y|}{\epsilon \tau \delta} \log \left(  \frac{ d_G \lambda |\Y|}{\epsilon \tau \delta} \right) \right)  \right)$ rounds with one query per round,
	\item[(b)] $m_t \leq O \left(\frac{1}{\tau} \log \frac{t}{\delta} \right)$ atoms drawn per round, and
	\item[(c)] $n_t \leq O \left(\left(\frac{\lambda^2}{\epsilon \rho} \right) \log \frac{(m_t + |\Y|)t}{\delta} \right)$ structures sampled per round.
\end{itemize}
\end{thm}
\begin{proof}
If we use the stopping criterion from Lemma~\ref{lem: stopping criterion} with the threshold $3\epsilon/4\lambda^2$, then at the expense of drawing an extra $\frac{48\lambda^2}{\epsilon} \log \frac{t(t+1)}{\delta}$ hypotheses for each round $t$, we are guaranteed that with probability $1 - \delta$ if we ever encounter a round $t$ in which $\avg(\pi_t) \leq \epsilon/(2\lambda^2)$ then we terminate and we also never terminate whenever $\avg(\pi_K) > \epsilon$. Thus if we do ever terminate at some round $t$, equation~\eqref{eqn: diameter bounds posterior distance} guarantees
\[ D(\pi_t, \nu_t) \ \leq \ \epsilon. \]

Note that if we draw $m_t \geq \frac{1}{\tau} \log \frac{t(t+1)}{\delta}$ atoms per round, then with probability $1-\delta$ one of them will $\rho$-average split $\pi_t$ if $\avg(\pi_t) > \epsilon/(2\lambda^2)$. Conditioned on this happening, Lemma~\ref{lem: select lemma} guarantees that that with probability $1-\delta$ {\sc select} finds a point that $\rho/2$-average splits $\pi_t$ while drawing at most $O\left(\frac{\lambda^2}{\epsilon \rho} \log \frac{(m_t + |\Y|)t(t+1)}{\delta}\right)$. 

If after $T$ rounds we still have not terminated, then $\avg(\pi_T) > \epsilon/(2\lambda^2)$. By Lemma~\ref{lem: Bayesian potential decrease} we also know
\[ Z_T^2 \, \avg(\pi_T) \ \leq \ \exp\left(-\rho(1-e^{-\beta})T/2 \right) \ = \  \exp\left(-\frac{\rho T}{2} \cdot\frac{1-2q}{1-q} \right). \]
By Lemma~\ref{lem: normalizing constant lower bounds}, we have that for all rounds $t \geq 1$, with probability $1-\delta$,
\[ \log\frac{1}{Z_t} \ \leq \ \log \frac{2t(t+1)}{\delta} + d_G \log \frac{em^{(t)}(|\Y| + 1)}{d_G} + t q \ln \frac{1-q}{q} + \sqrt{t \log \frac{4t(t+1)}{\delta}}\left(\log \frac{4t^2(t+1)}{\delta} \right) .\]
Where $m^{(t)}$ is the number of atoms sampled up to time $t$, which can be bounded as 
\[ m^{(t)} \ \leq \ \frac{t}{\tau} \log \frac{t(t+1)}{\delta} .
\]
Putting this together, we can conclude that $\avg(\pi_T) \leq \epsilon/(2\lambda^2)$ whenever
\begin{align*}
T \ \geq \ \max \frac{2}{\gamma} & \left\{  \sqrt{T \log \frac{4T(T+1)}{\delta}}\left(\log \frac{4T^2(T+1)}{\delta} \right), \right.\\
& \; \; \; \left. \log \frac{2T(T+1)}{\delta} + d_G \log\left( \frac{e(|\Y| + 1)}{d_G} \cdot \frac{T}{\tau} \log \frac{T(T+1)}{\delta}  \right) + \log \frac{2\lambda^2}{\epsilon} \right\} .
\end{align*}

Note that $T \geq \frac{2}{\gamma} \sqrt{T \log \frac{4T(T+1)}{\delta}} \left( \log \frac{4T^2(T+1)}{\delta} \right)$ whenever $T \geq \frac{4}{\gamma^2} \log^3\left( \frac{4 T^2 (T+1)}{\delta} \right)$ and this is satisfied for 
\[ T \geq \frac{4c_1}{\gamma^2} \left(\log^3 \frac{4}{\gamma^2} + \log^3 \frac{4}{\delta} \right) \] 
where $c_1 = 2^{22}$ suffices. 

Further, we have $T \geq \frac{2}{\gamma}\left( \log \frac{2T(T+1)}{\delta} +d_G \log\left( \frac{e(|\Y| + 1)}{d_G} \cdot \frac{T}{\tau} \log \frac{T(T+1)}{\delta}  \right) + \log \frac{2\lambda^2}{\epsilon} \right)$ is satisfied whenever we have $T \geq  \frac{2}{\gamma} \left( (1+d_G) \log \frac{2T(T+1)}{\delta} +d_G \log\left( \frac{e(|\Y| + 1)}{ \tau d_G}  \right) + \log \frac{2\lambda^2}{\epsilon} \right)$. We can achieve this with 
\[ T \geq \frac{2 c_2}{\gamma} \left(d_G \log \frac{e(|\Y| + 1)}{\tau d_G} + \log \frac{2\lambda^2}{\epsilon} + c_2 (1 + d_G) \log \left(\frac{4(1 + d_G)}{\gamma \delta} \left( d_G \log \frac{e(|\Y| + 1)}{\tau d_G} + \log \frac{2\lambda^2}{\epsilon}  \right)    \right)    \right) \]
where $c_2 = 50$ suffices.
\end{proof}

\end{document}